\pdfoutput=1
\documentclass[11pt]{article}
\usepackage{graphicx} 
\usepackage{diagbox}
\usepackage{authblk}
\makeatletter
\renewcommand\AB@authnote[1]{\textsuperscript{#1}\hspace{5pt}}

\makeatother
\usepackage{hyperref}
\usepackage{algorithm}
\usepackage{algpseudocode}
\usepackage{notation}
\usepackage{arxiv-2}
\usepackage{mathrsfs}
\usepackage{float}
\usepackage{setspace}
\usepackage{lmodern}        

\title{\normalfont Pre-Trained AI Model Assisted Online Decision-Making under Missing Covariates: A Theoretical Perspective}

\author{
\begin{minipage}[t]{0.3\textwidth}
    \centering
    Haichen Hu \\
    {\small\text{CCSE, MIT}} \\
    {\small\texttt{huhc@mit.edu}}
\end{minipage}
\hspace{0.12em}
\begin{minipage}[t]{0.3\textwidth}
    \centering
    David Simchi-Levi \\
    {\small\text{IDSS, MIT}} \\
    {\small\texttt{dslevi@mit.edu}}
\end{minipage}
}

\date{}
\begin{document}
\maketitle

\begin{abstract}
We study a sequential contextual decision-making problem in which certain covariates are missing but can be imputed using a pre-trained AI model. From a theoretical perspective, we analyze how the presence of such a model influences the regret of the decision-making process. We introduce a novel notion called \textit{``model elasticity''}, which quantifies the sensitivity of the reward function to the discrepancy between the true covariate and its imputed counterpart. This concept provides a unified way to characterize the regret incurred due to model imputation, regardless of the underlying missingness mechanism. More surprisingly, we show that under the missing at random (MAR) setting, it is possible to sequentially calibrate the pre-trained model using tools from orthogonal statistical learning and doubly robust regression. This calibration significantly improves the quality of the imputed covariates, leading to much better regret guarantees. Our analysis highlights the practical value of having an accurate pre-trained model in sequential decision-making tasks and suggests that model elasticity may serve as a fundamental metric for understanding and improving the integration of pre-trained models in a wide range of data-driven decision-making problems.

\end{abstract}
\vspace{-0.5em}
\noindent\rule{\textwidth}{1pt} 
\vspace{1em}
\section{Introduction}\label{sec:intro}
In a wide range of operational settings, including personalized pricing \citep{bu2022context}, healthcare decision support \citep{ameko2020offline}, and digital platform management \citep{he2020contextual}, decision-makers are increasingly required to adaptively learn from limited feedback while personalizing decisions to user- or context-specific information. A common modeling framework for such problems is the contextual bandit, in which a decision-maker observes covariates (or context) associated with a decision instance, selects an action, and then receives partial feedback, typically in the form of a realized reward or cost \citep{wang2025dynamic,keyvanshokooh2025contextual,li2025online}. The aim is to learn a policy that balances exploration and exploitation while adapting to the contextual information over time.

However, in many practical applications, the full set of contextual features is often only partially observable. Customers may choose to withhold demographic or behavioral information for privacy reasons \citep{martin2017data,bracale2025dynamic}; in healthcare systems, historical records may be incomplete or fragmented \citep{dziura2013strategies,turkson2021handling}; and in operational platforms, data pipelines may fail to record certain covariates. The presence of missing covariates renders the standard contextual bandit formulation ill-posed: the decision-maker lacks key information required to condition actions on relevant states, potentially leading to biased learning and degraded performance. While one might consider discarding instances with incomplete context, doing so may substantially reduce data availability, particularly in high-frequency decision environments.

To mitigate the impact of missing covariates, it is natural to consider auxiliary sources of information that may serve as proxies or surrogates. Recent advances in artificial intelligence—particularly the use of large-scale pre-trained models and LLMs—have enabled the generation of predictions or synthetic features at low marginal cost for science and engineering \citep{jumper2021highly,irvin2019chexpertlargechestradiograph}. For instance, in a digital pricing or online recommendation system, a language model trained on historical customer interactions may be used to infer likely missing attributes based on partially filled forms or clickstreams \citep{yu2025application,zhang2025contextualthompsonsamplinggeneration,cai2025activeexplorationautoregressivegeneration}. These AI-generated predictions, however, are not ground truth: they may be biased, miscalibrated, or trained under a different data regime.  This raises a natural question:
\begin{center}
\emph{Is it possible to incorporate pre-trained AI surrogate predictions into the downstream decision-making framework and systematically correct for their potential bias or miscalibration using observed feedback over time?}
\end{center}
\begin{figure}[ht]
    \begin{center}
\centerline{\includegraphics[width=0.5\linewidth]{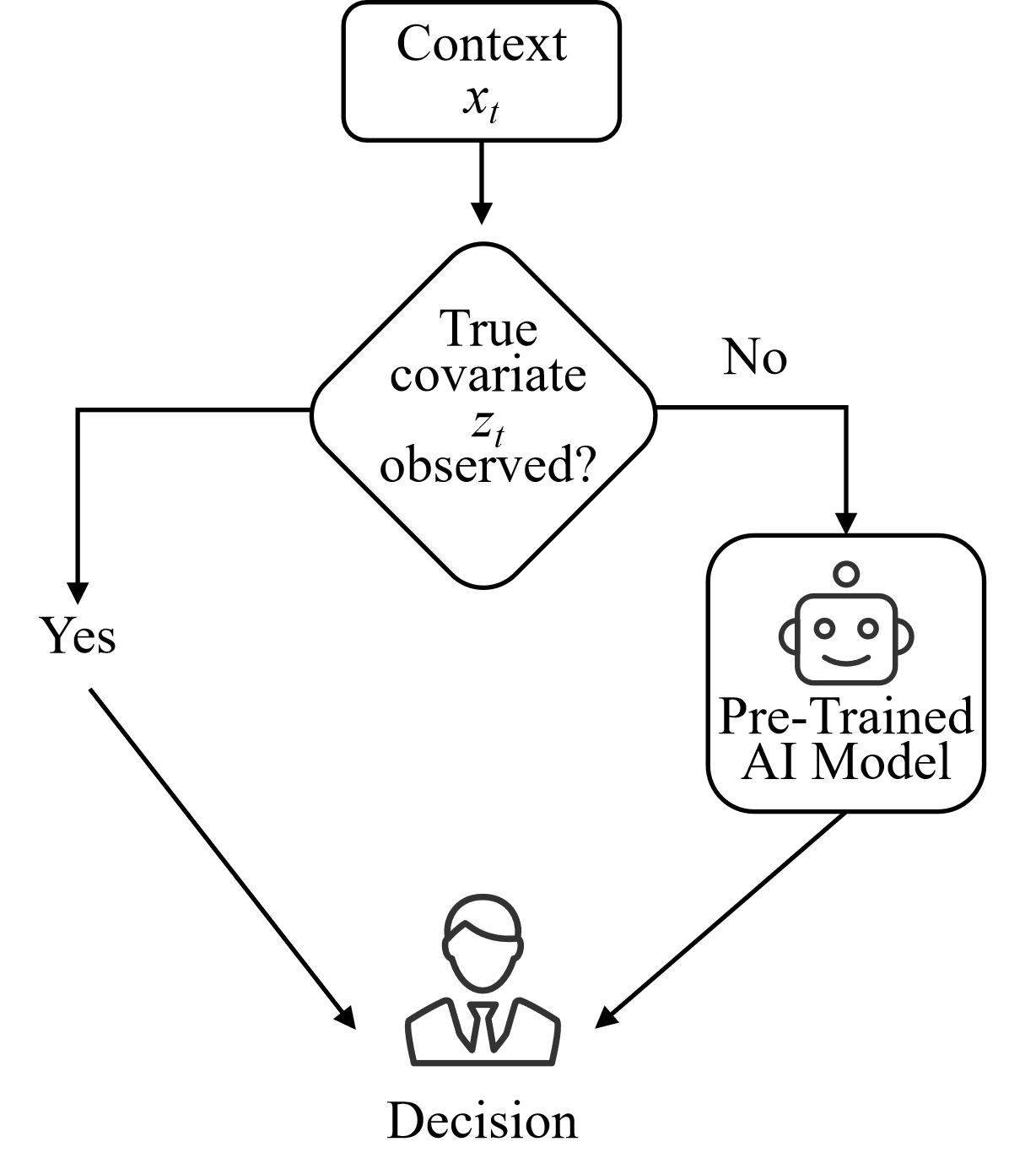}}
\caption{Pre-trained model assisted sequential decision-making with missing covariates\label{fig:illustration}.}
\end{center}
\vspace{-5 mm}
\vskip -0.2in
\end{figure}
\subsection{Our Contribution}
In this paper, we study the popular problem about the AI assisted downstream decision-making problem from a theoretical side. Specifically, we consider a sequential decision-making problem in which the decision-maker observes only partial context but has access to a pre-trained model that provides predictions about the missing features. Our objective is to develop a principled framework for incorporating such auxiliary information into the learning process while ensuring long-run performance guarantees. Beginning with the most general Missing Not At Random (MNAR) framework, we characterize the risk of reward estimation using pre-trained model imputation, and show that it decomposes into a standard oracle risk term and an additional term reflecting the quality of the surrogate model—what we refer to as ``model elasticity with respect to the pseudo response''. This decomposition provides insight into how the auxiliary predictions interact with the learning algorithm and affect cumulative performance.

We further consider settings in which the missingness mechanism satisfies a Missing At Random (MAR) condition, allowing the bias in the pre-trained model to be sequentially corrected using observed data. Under this assumption, we can dynamically calibrate the pre-trained model that integrates both real and surrogate information through orthogonal statistical learning \citep{foster2023orthogonalstatisticallearning}, and analyze the resulting performance in terms of learning efficiency and long-run regret. Our results suggest that pre-trained models, when properly calibrated, can significantly improve decision quality in contextual environments with missing information, offering a promising direction for robust data-driven operations.

\subsection{Roadmap of the Paper}
We present our results in the following structure. In \cref{sec:model}, we introduce the problem setting of sequential decision-making with missing covariates and a pre-trained model for imputation. In \cref{sec:algorithm}, we propose a general algorithmic framework and provide its statistical guarantee based on our newly introduced concept of model elasticity. Finally, in \cref{sec:model_calibration}, we focus on the missing at random (MAR) setting and develop an efficient calibration procedure for the pre-trained model, leading to improved regret bounds.

\section{Related Works}\label{sec:related work}
Our work is closely connected to several lines of research. First, the literature on \textbf{bandits with partially observed data} addresses learning under missing data without relying on pre-trained models. Second, studies on \textbf{regression imputation} and \textbf{surrogate-aided prediction} provide theoretical and empirical support for imputing pseudo-responses using pre-trained models in estimation and prediction tasks. Third, advances in \textbf{synthetic data generation} highlight the growing accessibility of large-scale data via pre-trained AI agents, offering an inexpensive alternative to costly human annotation. Finally, the frameworks of \textbf{orthogonal statistical learning} and \textbf{debiased machine learning} offer powerful methodologies for calibrating pre-trained models and mitigating bias. In what follows, we review each of these research directions in greater detail.


\paragraph{Contextual Bandits with Partially Observed Data:} Learning with missing data is an important subfield in machine learning and operations research. \citet{bouneffouf2020contextual,guinet2022effective} studied general bandits with missing reward signals and characterized the regret loss due to reward censorship. \citet{kim2025linear} investigated linear bandits with latent unobserved features. From operations research, several papers have considered online learning with censored reward observations in inventory control and dynamic pricing \citep{chen2022dynamic,chen2024optimal,ding2024feature,zhang2025thompsonsamplingrepeatednewsvendor}. However, there are all special structures in their models which make the data censorship tractable. The investigation about bandits with missing covariates is much more limited, \citet{jang2022high} considered linear bandits in a missing completely at random (MCAR) setting.  In our paper, given a pre-trained model, we consider a much more challenging learning with missing covariates under MNAR without special model assumptions. 

\paragraph{Regression Imputation and Surrogate-Aided Prediction:}
As a classic topic in economics, statistics, and business \citep{little2019statistical,bell2007bellkor}, regression imputation is about filling in missing covariate values by predicting them from observed variables using a fitted regression model. \citet{greenlees1982imputation} empirically verified the efficiency of stochastic imputation of nonresidents' salary data. \citet{breunig2021nonparametric} focused on regression under selectively observed covariates with instrumental variables and studied the association between income and health as experiments. Furthermore, \citet{drechsler2025imputationstrategiesrightcensoredwages} evaluates Tobit and quantile regression as imputation strategies to address top-coded (right-censored) wages in longitudinal administrative datasets. More recently, \citet{xia2024predictionaidedsurrogatetraining} studied empirical risk minimization with true label censorship and simulated pseudo-response, and \citet{dolin2025statistically} applied this method to post-deployment monitoring in digital health. By contrast, in this paper, the pre-trained model generated data are used as imputation of the missing covariates for training the estimated reward function. 
\paragraph{AI Synthetic Data Generation and Pseudo-Labelling:} In AI society, synthetic data generation and pseudo-labeling leverage pre-trained models to assign labels to data, enabling efficient training in low-resource or semi-supervised settings by augmenting or expanding the labeled dataset \citep{pham2021meta,nadas2025synthetic}. They are especially useful for augmenting low-resource tasks where real collection is expensive or impractical. People have been using LLMs for generating synthetic data tailored to specific tasks such as intent classification \citep{sahu2022dataaugmentationintentclassification}, and clinical text mining \citep{tang2023does}. Pseudo-labeling is used for dialog summarization and other tasks \citep{mishra2023llm,ran2024pseudo} via semi-supervised learning. Similarly, in this paper, we view the pre-trained model as a generative model for pseudo-responses and utilize it to help our sequential decision-making task. In our paper, we studied the influence of the quality of the pseudo-response generated by the pre-trained model on our decision-making problem.

\paragraph{AI Assisted Decision-Making:}
People have been using generative AI and LLMs for downstream decision-making in many applications including online platforms, clinical trials, and heuristic optimization. \citet{ye2025lola} introduces a hybrid algorithm that integrates large language model (LLM) predictions with online learning to optimize content experiments. \citet{cao2024hr} developed a human-machine collaboration method in linear bandits with a resource constraint. Moreover, in sequential model calibration, \citet{collina2024tractableagreementprotocols,collina2025collaborative} studied how to achieve consensus on predictions and improve accuracy by reducing any ML algorithm to an interactive protocol with humans. Notably, \citet{cai2025activeexplorationautoregressivegeneration, zhang2025contextualthompsonsamplinggeneration} proposed a new approach to active exploration in Thompson sampling by replacing posterior sampling of latent parameters with autoregressive generation of missing outcomes.

\paragraph{Orthogonal Statistical Learning and Debiased Machine Learning:} Inspired by the doubly robust estimation \citep{funk2011doubly}, \citet{chernozhukov2017double} and \citet{foster2023orthogonalstatisticallearning} proposed a debiased machine learning and orthogonal statistical learning, which uses an accurate estimation of the nuisance parameter and cross-fitting to help approximate the true parameter of interest. This framework was then extended for various settings including covariate shifting \citep{chernozhukov2023automatic}, high dimensionality \citep{chen2025automaticdoublyrobustforests}, heterogeneous causal effect estimation \citep{kennedy2024minimax}, and so on. In this paper, we utilize orthogonal statistical learning to conduct the calibration procedure in the decision-making algorithm to improve the pre-trained model sequentially.
\subsection{Notations}
Throughout the paper, we adopt standard asymptotic notation. When we write $a = \cO(b)$ or $a \lesssim b$, we mean that there exists a universal constant $c > 0$ such that $a \leq c \cdot b$. Similarly, $a = \Omega(b)$ or $a \gtrsim b$ indicates that $a \geq c \cdot b$ for some constant $c > 0$. We use the notation $a \asymp b$ to denote that $a$ and $b$ are of the same order, i.e., both $a \lesssim b$ and $a \gtrsim b$ hold. The notation $\Tilde{\cO}(\cdot)$ is used to suppress logarithmic factors in asymptotic expressions.

Following standard notation in empirical process theory, we denote by $\PP(f)$ the population expectation $\EE_{X \sim \PP}[f(X)]$ of a measurable function $f$, where $\PP$ is a probability distribution over the input space. Letting $\PP_N$ be the empirical distribution based on $N$ i.i.d. samples $\cbr{X_i}_{i=1}^N$, we write $\PP_N(f) := \frac{1}{N} \sum_{i=1}^N f(X_i)$ for the empirical average. For any square-integrable function $f$, we define its $L_2$-norm with respect to $\PP$ as $||f||_2:=(\int f(X)^2d\PP)^{1/2}$ and its empirical $L_2$-norm (also referred to as a pseudo-norm) as $||f||_N:=\frac{1}{N}\sum_{i=1}^{N}f(X_i)^2$.

\section{Problem Setup}\label{sec:model}

In this section, we describe the interaction protocol and formalize the learning environment. Let the context space be denoted by $\cX \subset \RR^d$. Suppose the interaction unfolds over $T$ rounds, indexed by $t \in [T]$. At each round $t$, the decision maker (DM) observes a context vector $x_t \in \cX$ which is sampled i.i.d. from an unknown distribution $\PP_{x}$. 

In addition to the observed context, there exists an auxiliary covariate space $\cZ \subset \RR$. At each round, the true value of the covariate, denoted by $z_t^*$, which is sampled i.i.d. from $\PP_{z^*}$, may be unobserved—i.e., missing from the DM’s view. We use a Bernoulli random variable \(b_t\) to indicate the missingness of the covariate at time \(t\), where \(b_t = 1\) denotes that the covariate is observed, and \(b_t = 0\) indicates that it is missing. $b_t$ could be related to both $x_t$ and $z_t^*$, and we use $\PP_{x,z^*,b}$ to denote the joint distribution of them.

The DM also has access to a finite action set $\cA = \cbr{a_1, \cdots, a_K}$. Upon selecting an action $a_t \in \cA$ at round $t$, the DM receives a reward signal $r_t$, which is modeled as
\[
r_t = f^*(x_t, z_t^*, a_t) + \xi_t,
\]
where $\xi_t$ denotes a zero-mean random noise. For simplicity, we assume that $\xi_t$ is bounded in $[-\lambda,\lambda]$. Following \citep{saha2022efficient,han2023optimal,hu2025contextual}, we impose the following realizability assumption.
\begin{assumption}
    We have access to a known function class $\cF$ such that $f^*\in\cF$. Moreover, $\cF$ is convex and thus for any $f\in\cF$, $\cF-f$ is star-shaped.
\end{assumption}

In the absence of side information, the DM is unable to make effective decisions in rounds where $z_t^*$ is unobserved, as this prevents accurate evaluation of the underlying reward function. To address this, we assume that the DM has access to a pre-trained model $\tilde{g}$—obtained from prior datasets or domain knowledge—that provides a prediction of $z$ given the observed context $x$. This setting captures practical scenarios where sensitive covariates such as income, weight, or medical history may be unavailable during decision-making, but predictive surrogates are accessible.

The DM may employ a stochastic policy, where the action $a_t$ is sampled probabilistically. Formally, we define a policy $\pi(\cdot \mid x, z)$ as a mapping from $\cX \times \cZ$ to the probability simplex $\Delta(\cA)$ over actions. The goal of the DM is to maximize the total expected reward over $T$ rounds. Accordingly, the performance of any policy sequence $\cbr{\pi_t}_{t=1}^{T}$ is evaluated through its regret, defined as
\[
\text{Reg}(\cbr{\pi_t}_{t=1}^{T}) := \sum_{t=1}^{T} \EE_{\substack{x_t, z_t^* \\ a_t \sim \pi^*}} \sbr{f^*(x_t, z_t^*, a_t)} - \EE_{\substack{x_t, z_t^* \\ a_t \sim \pi_t}} \sbr{f^*(x_t, z_t^*, a_t)},
\]
where $\pi^*$ denotes an oracle policy that has full access to the true covariates $z_t^*$ at each round.

For simplicity, we assume that all functions in $\cF$ take values in $[0, 1]$, and that the additive noise $\xi_t$ is almost surely bounded in the interval $[-\lambda, \lambda]$ for some known constant $\lambda > 0$.

\section{Algorithmic Framework}\label{sec:algorithm}
In this section, we introduce our general algorithmic framework for contextual bandits in the presence of covariates that are missing not at random (MNAR). We begin by outlining the structure and motivation of our approach and then provide formal descriptions of the stochastic policy and algorithm. Finally, we preview the theoretical guarantee that characterizes the regret performance of our method.
\subsection{Pre-Trained Model Assisted Decision-Making Algorithm}\label{subsec:algorithm}

Our framework is motivated by the observation that the functional dependence of the reward on unobserved covariates can significantly complicate the decision-making process. In particular, we find that the \emph{elasticity} of the reward function class $\cF$ with respect to the potentially missing covariate $z$—that is, how sensitive the reward is to perturbations in $z$—plays a crucial role in determining the statistical performance and regret rate of the proposed algorithm.

\vspace{0.5em}
\noindent\textbf{Algorithmic Idea.} The core idea of our approach is simple yet broadly applicable. Under the MNAR setting, even though the covariate $z_t^*$ may be unavailable at round $t$, the DM can rely on a pre-trained predictor $\tilde{g}$ that estimates $z_t^*$ based on the observed context $x_t$. This predictor may be trained using historical data where both $x$ and $z$ were observed. Once $\tilde{g}(x_t)$ is obtained, we simply treat it as a plug-in estimator for the missing covariate and proceed with decision-making as if $z_t^* = \tilde{g}(x_t)$.

This plug-in strategy has several desirable properties. First, it is agnostic to the missingness mechanism—no explicit model of the probability of missingness is needed. Second, it is adaptive to the quality of the pre-trained model: when $\tilde{g}$ is accurate, the plug-in estimate leads to effective decisions; when $\tilde{g}$ is biased or miscalibrated, the algorithm can still achieve robust performance through online feedback and policy updates.

\vspace{0.5em}
\noindent\textbf{Stochastic Policy: Inverse Gap Weighting.} To ensure both exploration and exploitation, we adopt the \emph{inverse gap weighting} (IGW) policy, which has been proposed and analyzed in recent works \citep{foster2020beyond, simchi2020bypassing, qian2024offline}. Given a reward function estimate $\hat{f} \in \cF$ and a tuning parameter $\gamma > 0$, the IGW policy assigns action probabilities inversely proportional to the reward gaps. Formally, for any context-covariate pair $(x, z)$ and any action $a \in \cA$, we define
\[
\texttt{IGW}_{\gamma,\hat{f}}(a \mid x, z) = 
\begin{cases}
  \frac{1}{K + \gamma \left( \hat{f}(x, z, \hat{a}_{x, z}) - \hat{f}(x, z, a) \right)} & \text{if } a \neq \hat{a}_{x, z}, \\
  1 - \sum_{a' \neq \hat{a}_{x, z}} \texttt{IGW}_{\gamma,\hat{f}}(a' \mid x, z) & \text{if } a = \hat{a}_{x, z},
\end{cases}
\]
where $\hat{a}_{x, z} := \argmax_{a \in \cA} \hat{f}(x, z, a)$ denotes the action that maximizes the estimated expected reward.

This policy promotes exploration by ensuring that suboptimal actions are still played with positive probability, but the probability mass concentrates around actions with smaller reward gaps—an appealing balance between exploration and exploitation.

\vspace{0.5em}
\noindent\textbf{Full Algorithm Description.} Building on this stochastic policy, we now present the full online learning algorithm in \cref{alg:IGW_pseudo_code}. The algorithm proceeds in epochs, indexed by $m = 1, 2, \dots$, and uses a growing dataset to update reward estimates periodically. In each epoch, the algorithm uses a black-box learning oracle, denoted by $\Alg$, to fit a reward function $\hat{f}_m$ from past data collected in previous rounds. The dataset consists of tuples of the form $(x_t, b_t, \tilde{z}_t, a_t, r_t)$, where $\tilde{z}_t$ is either the true $z_t^*$ if observed, or the plug-in estimate $\tilde{g}(x_t)$ otherwise. We name our algorithm ``$\mathsf{PRIMO}$'' standing for ``Pre-tRained model assisted decIsion-making with Missing cOvariates''.

\vspace{0.5em}
\noindent\cref{alg:IGW_pseudo_code} summarizes the key steps of the proposed decision-making procedure. For the algorithm structure, we first divide the whole $T$ rounds
into several epochs and geometrically increase the length of every epoch so that the low-frequency oracle call
property is automatically satisfied in \cref{alg:IGW_pseudo_code}. At the beginning of every epoch m, we call to fit the reward function based on the i.i.d. data gathered from the last epoch.
We then design our inverse gap weighting policy $\texttt{IGW}$ based on the estimated reward and execute it throughout this epoch. To be specific, we first impose an epoch schedule $\beta_s=2^s$, which means that the $s+1$ th epoch is twice as long as the previous one. Therefore, the statistical guarantees we get from later epochs are stronger, and as $s$ scales, our estimation becomes more and more accurate. The Inverse Gap Weighting
policy enables us to balance the exploration and exploitation trade-off to maintain a low regret just assuming access to an offline regression oracle.

\begin{algorithm}[ht]
\caption{Pre-tRained model assisted decIsion-making with Missing cOvariates ($\mathsf{PRIMO}$)}\label{alg:IGW_pseudo_code}
\begin{algorithmic}
    \State \textbf{Require:} Context space $\cX$; covariate space $\cZ$; action space $\cA$; function class $\cF$; pre-trained model $\tilde{g}$; tuning parameters $\cbr{\gamma_s}_{s=1}^{\infty}$
    \State Initialize epoch schedule $0 = \beta_0 < \beta_1 < \beta_2 < \cdots$
    \For{epoch $s = 1, 2, \cdots$}
        \State Construct dataset $\cD_{s-1} = \cbr{(x_t,b_t, \tilde{z}_t, a_t, r_t)}_{t = \beta_{s-2}+1}^{\beta_{s-1}}$
        \State Fit reward function $\hat{f}_s = \Alg(\cD_{s-1})$
        \For{round $t = \beta_{s-1} + 1, \dots, \beta_s$}
            \State Observe context $x_t$
            \State If $z_t^*$ is observed, set $\tilde{z}_t = z_t^*$; otherwise set $\tilde{z}_t = \tilde{g}(x_t)$
            \State Define policy $\texttt{IGW}_{\gamma_s, \hat{f}_s}$:
            \begin{equation*}
                p_t(a) =
                \begin{cases}
                    \frac{1}{K + \gamma_s \left( \hat{f}_s(x_t, \tilde{z}_t, \hat{a}^s_{x_t,\tilde{z}_t}) - \hat{f}_s(x_t, \tilde{z}_t, a) \right)} & \text{if } a \neq \hat{a}^s_{x_t,\tilde{z}_t} \\
                    1 - \sum_{a \neq \hat{a}^s_{x_t,\tilde{z}_t}} p_t(a) & \text{if } a = \hat{a}^s_{x_t,\tilde{z}_t}
                \end{cases}
            \end{equation*}
            \State Sample action $a_t \sim p_t$ and receive reward $r_t$
        \EndFor
    \EndFor
\end{algorithmic}
\end{algorithm}

\vspace{0.5em}
\noindent In the next subsection, we formally define the notion of \emph{model elasticity} with respect to the covariate $z$, which captures the sensitivity of the reward function to covariate perturbations. We then specify the oracle algorithm $\Alg$ used for reward function estimation, and provide a theoretical analysis of the resulting regret. In particular, we establish an oracle inequality that quantifies the excess risk of the estimated reward function, and we derive appropriate choices of tuning parameters $\cbr{\gamma_s}_{s=1}^\infty$ to balance exploration and estimation accuracy.

\subsection{Empirical Risk Minimization with Pre-Trained Model Imputation}\label{subsec:offline_ERM_oracle}

In this subsection, we formally define our empirical risk minimization oracle, denoted by \(\Alg\), which is used to estimate the reward predictors based on the i.i.d. dataset \(\cD = \{(x_i, b_i, \tilde{z}_i, a_i, r_i)\}_{i=1}^N\). Here, each tuple in \(\cD\) consists of an observed context \(x_i\), the missing indicator $b_i$, a possibly imputed covariate \(\tilde{z}_i\), the chosen action \(a_i\), and the corresponding observed reward \(r_i\).

We emphasize that the data collected in each epoch are independent and identically distributed (i.i.d.). This follows from the assumption that the underlying context-covariate pair \((x_t, z_t^*)\) is drawn i.i.d. in each round from the unknown distribution $\PP_{x,z^*}$. Although the missingness mechanism may depend on both \(x_t\) and \(z_t^*\), the resulting missingness indicator $b_t$ remains i.i.d.\ across rounds due to the independence of the base data-generating process. Consequently, after applying a fixed imputation rule (e.g., using a pre-trained model), the constructed covariate \(\tilde{z}_t\) also yields an i.i.d.\ sequence, ensuring that the joint imputed pair \((x_t, \tilde{z}_t)\) maintains the i.i.d.\ property. With a little abuse of the notation, we use $\PP_{\cD}$ and $\PP$ interchangeably to denote the data distribution in $\cD$ when developing the oracle inequalities. 

Moreover, our stochastic policy used for action selection remains fixed across all epochs. This further ensures that the full data tuple \((x_t, b_t, \tilde{z}_t, a_t, r_t)\) observed in each round is sampled independently and identically from a fixed distribution. These properties collectively justify the application of standard statistical learning tools in each epoch for reward model estimation via empirical risk minimization.

We rearrange the dataset $\cD$ in the following way.
\[
\cD=\cbr{(x_i,b_i,\tilde{z}_i,a_i,r_i)}_{i=1}^{N}=\cbr{(x_i,1,z_i^*,a_i,r_i)}_{i=1}^{m}\cup\cbr{(x_i,0,\tilde{g}(x_i),a_i,r_i)}_{i=m+1}^{n},
\]
where $N=m+n$. That is to say, in dataset $\cD$, we have $m$ uncensored covariates and $n$ missing ones. For every $1\le i\le N$, we use $b_i$ as the Bernoulli random variable to represent whether $z^*$ is missing or not, i.e., $b_i=1$ means observing the true covariate $z^*$ and $b_i=0$ means the covariate is missing. Therefore, we could rewrite our dataset $\cD$ as
\[
\cD=\cbr{(x_i,b_i, b_iz_i^*+(1-b_i)\tilde{g}(x_i),a_i,r_i)}_{i=1}^{N}.
\]

We then solve the following ERM problem after collecting $\cD$. Specifically, we denote the ERM procedure as $\Alg$.
\begin{align}\label{ERM}
\hat{f}\in\argmin_{f\in\cF}\cbr{\frac{1}{N}\sum_{i=1}^{N}\rbr{f(x_i,\Tilde{z}_i,a_i)-r_i}^2}.
\end{align}

We now want to characterize the statistical property of $\hat{f}=\Alg(\cD)$. To do that, we first define the local Rademacher complexity and critical radius of function classes in empirical process theory. 

\begin{definition}[Local Rademacher Complexity]
    For the observed data generating distribution $\mathit{\PP}_{\cD}$ and any positive number $r$, the local Rademacher complexity $\cR_{N}(t,\cF)$ is
    \[
    \cR_{N}(t;\cF)=\EE_{\varepsilon,\cD}\sbr{\sup_{\substack{f\in\cF,\\||f-f^*||_2\le t}}\frac{1}{N}\sum_{i=1}^{N}\varepsilon_i\rbr{f(x_i,\Tilde{z}_i,a_i)-f^*(x_i,\Tilde{z}_i,a_i)}},
    \]
    where $||h||_2=(\int h^2d\mu_{\cD})^{1/2}$. $\mu_{\cD}$ is the marginal probability measure induced by $\PP_{\cD}$.
\end{definition}
\begin{definition}[Critical Radius]
    We define the critical radius as 
    $r_{N}=\argmin\cbr{t: \frac{t}{16}\ge \frac{\cR_N(t;\cF)}{t}}$.\footnote{More concretely, in this paper, $r_N$ is used to denote the critical radius for reward function class. We will use other letter to represent the critical radius for other function classes.}
\end{definition}
As long as the function class is star-shaped, we know that function $\delta\mapsto\frac{\cR_{N}(t,\cF)}{t}$ is decreasing \citep{wainwright2019high}. The critical radius and local Rademacher complexity characterize the learnability difficulty of the function class. More importantly, critical radius $r_N$ corresponds to a type of best oracle risk that could be achieved if we had access to $N$ labeled samples \citep{bartlett2005local,koltchinskii2011oracle}. We claim that the oracle risk term is typically dominant in statistical learning theory problems. Indeed, apart from scalar parametric problems, we have $r_N\gtrsim N^{-1/2}$. We present some examples in \cref{example:critical_radius}.
\begin{example}\label{example:critical_radius}
For illustration, we could compute the critical radius of some function classes. More examples could be found in \citet{wainwright2019high} and we omit them here.
    \begin{itemize}
        \item Linear regression with dimension $d$, $r_N\lesssim \sqrt{\frac{d}{N}}$.
        \item For Lipschitz regression, $\cF=\cbr{f:[0,1]\mapsto \RR, f(0)=0,f\ \text{is $L$-Lipschitz}}$, $r_N\lesssim\rbr{\frac{L}{N}}^{1/3}$.
        \item For convex regression, $\cF=\cbr{f:[0,1]\mapsto \RR, f(0)=0,f\ \text{is convex and $L$-Lipschitz}}$, $r_N\lesssim\rbr{\frac{L}{N}}^{2/5}$.
        \item For twice-differentiable functions, $\cF=\cbr{f:[0,1]\rightarrow\RR:||f||_{\infty}+||f'||_{\infty}+||f''||_{\infty}\le c_0<\infty}$, then we have $r_N\lesssim \frac{C(c_0)}{N^{2/5}}$ where $C(c_0)$ is some constant related to $c_0$.
    \end{itemize}
\end{example}
After introducing our empirical risk minimization (ERM) procedure, it is important to observe that the reward $r_t$ is always generated from the true context covariate pair $(x_t,z_t^*)$. In contrast, our regression model relies on the imputed covariate 
$\tilde{z}_t$, which is provided by a pre-trained model $\tilde{g}$. This discrepancy raises the need to carefully analyze how the quality and bias of the pre-trained model $\tilde{g}$ affect the accuracy of the estimated reward function. To capture this influence, we introduce a notion closely related to the concept of elasticity from economics, which we elaborate on in \cref{subsec:model_elasticity}.

\subsection{Model Elasticity with respect to the Covariate Variable}\label{subsec:model_elasticity}

In this subsection, we introduce a key metric for evaluating the quality of the pre-trained model: the \emph{model elasticity}. Our motivation is inspired by the notion of elasticity in economics, which quantifies how sensitive an outcome—such as demand or supply—is to changes in an underlying variable, such as price. A high elasticity indicates that small changes in the input can lead to large fluctuations in the output. Analogously, in our setting, we define the elasticity of the reward function with respect to the missing covariate as a measure of its sensitivity to variation in that unobserved input. This analogy provides a useful lens for assessing the importance of accurate imputation: when the model elasticity is high, even slight errors in the imputed covariate can lead to significant deviations in predicted rewards, thereby impacting decision quality.

We formally define the model elasticity in \cref{def:model_elasticity}. Let \(\PP(z^*\mid x)\) denote the true conditional distribution of the covariate \(z\) given the context \(x\). Taking a supremum over $f\in\cF$ and $a\in\cA$, we define the following quantity.
 
\begin{definition}\label{def:model_elasticity}
    For the function class $\cF$, and the true covariate distribution $P^*(\cdot|x)$, the model elasticity of the pseudo-response is defined as
    $$\cE_{\PP}^{\cF}(\tilde{g}):=\sup_{\substack{f\in\cF,\\a\in\cA}}\EE_{x\sim \PP_x, z^*\sim \PP\rbr{\cdot|x}}\sbr{(f(x,z^*,a)-f(x,\tilde{g}(x),a))^2}.$$
\end{definition}

The quantity \(\cE^{\cF}_{\PP}(\tilde{g})\) serves as an analogue of \emph{elasticity} in economics, where elasticity traditionally quantifies how sensitive one variable (e.g., demand) is to changes in another (e.g., price). Similarly, in our setting, \(\cE^*_{\cF}(\tilde{g})\) measures how sensitive the predicted reward is to errors in the imputed covariate produced by the pre-trained model \(\tilde{g}\). Specifically, it quantifies the expected squared difference between the true reward (evaluated with the actual but unobserved covariate \(z^*\)) and the pseudo-reward computed using the imputed value \(\tilde{g}(x)\). A smaller value of this elasticity implies that the function class \(\cF\) is relatively robust to imputation error, while a larger value indicates that even small inaccuracies in \(\tilde{g}(x)\) can lead to large distortions in predicted rewards. Thus, this notion of elasticity provides a meaningful and interpretable way to evaluate the quality of the pre-trained model \(\tilde{g}\) relative to the decision-relevant reward function class. Since there is randomness in the covariate $z^*$, we assume that for any $f\in\cF$ $a\in\cA$, 
$$\EE_{x\sim \PP_x, z^*\sim \PP\rbr{\cdot|x}}\sbr{(f(x,z^*,a)-f(x,\tilde{g}(x),a))^2}:=\upsilon(f,\PP,\tilde{g})>0.$$ 
This assumption eliminates the possibility that the covariate $z$ has no influence on the reward value, and it is unnecessary and useless to consider regressing $z$ in our learning problem. We argue that this condition is easily satisfied in practice. For instance, when $\cF$ is a linear function class, we have $\upsilon(f,\PP,\tilde{g}) \ge \Omega\left(\EE\left[(z^* - \tilde{g}(x))^2\right]\right) > 0$, which holds as long as the pre-trained model $\tilde{g}$ does not perfectly recover the true covariate.

\subsection{ERM Oracle Inequality with Pre-Trained Model Imputation}\label{subsec:ERM Oracle Inequality with Imputation}

We now present the oracle inequality for the empirical risk minimization (ERM) estimator \(\hat{f}\) developed from \cref{subsec:offline_ERM_oracle}, given access to a dataset with imputed covariates from a pre-trained model. This result, stated in \cref{thm:erm_oracle_ineq}, quantifies how the estimation error depends both on the statistical complexity of the function class and on the quality of the imputed covariates.

For notational simplicity, we define the function \(\varphi_{a}(s)\) as \(\varphi_{a}(s) := \log_2(4a/s)\).

\begin{theorem}\label{thm:erm_oracle_ineq}
   Given the context space \(\cX\), covariate space \(\cZ\), action space \(\cA\), and a reward model class \(\cF\) that is 1-uniformly bounded, consider any i.i.d.\ dataset \(\cD = \{(x_i,b_i, \tilde{z}_i, a_i, r_i)\}_{i=1}^{N}\). Then for any \(\delta \in (0,1)\), with probability at least \(1 - 4\delta\),
    \begin{align*}
    \|\hat{f} - f^*\|_2 \le &\cO\Big((\max\{4, \lambda\} + 1)r_N+ \sqrt{\cE^{\cF}_{\PP}(\tilde{g})} + 4\max\{4, \lambda\} \sqrt{\frac{2\log(\varphi_{\lambda}(r_N)/\delta)}{N}}\\
    &+ \frac{\log(\varphi_{\lambda}(r_N)/\delta)}{r_N N} +\frac{4}{\upsilon(f^*,\PP,\tilde{g})^{1/2}}\sqrt{\frac{\log(2/\delta)}{N}}\Big).
    \end{align*}
     Consequently, applying the AM-GM inequality, we have that for any $\delta\in(0,1)$, with probability at least $1-\delta$,
    \[
    ||\hat{f}-f^*||_2^2\le C(\lambda,\upsilon(f^*,\PP,\tilde{g}))\rbr{r_N^2+\cE^{\cF}_{\PP}(\tilde{g})+\frac{2\log(\varphi_{\lambda}(r_N)/\delta)}{N}+\frac{\log^2(\varphi_{\lambda}(r_N)/\delta)}{(r_NN)^2}+\frac{\log(2/\delta)}{N}}.
    \]
\end{theorem}

This bound consists of three primary components. The first part involves $r_N$, reflects the standard oracle estimation error, as commonly seen in oracle inequalities for empirical risk minimization. The second term involves $\sqrt{\cE^{\cF}_{\PP}(\tilde{g})}$, captures the additional error introduced by using imputed covariates from the pre-trained model \(\tilde{g}\), as defined by the model elasticity in \cref{def:model_elasticity}. This term quantifies how sensitive the function class \(\cF\) is to imputation error, effectively serving as a penalty for covariate misspecification. The remaining terms are higher-order fluctuation terms.

For example, if the complexity radius is chosen as \(r_N \asymp N^{-p_1}\), and the pre-trained model satisfies \(\cE^*_{\cF}(\tilde{g}) \asymp N^{-p_2}\) for some \(0<p_1,p_2<1/2\), then the estimation error scales as
\[
\|\hat{f} - f^*\|_2 \lesssim \frac{1}{N^{p_1/2}} + \frac{1}{N^{p_2/2}},
\]
which illustrates that the first term corresponds to the usual statistical estimation error, while the second term reflects the quality of pseudo-response imputation. As the pre-trained model becomes more accurate (i.e., larger \(p\)), the second term becomes negligible, recovering near-optimal learning rates.
\subsection{Statistical Guarantee of \cref{alg:IGW_pseudo_code}}\label{subsec:regret}
Leveraging the oracle inequality established in \cref{thm:erm_oracle_ineq}, the inverse gap weighting (IGW) policy offers an efficient mechanism for translating oracle inequalities into statistical regret guarantees in sequential decision-making settings \citep{simchi2020bypassing,hu2025contextual}. Building on this connection, we derive the following regret bound for \cref{alg:IGW_pseudo_code}.

\begin{theorem}\label{thm:regret}
    For \cref{alg:IGW_pseudo_code}, if we set $\beta_s=2^{s}$ in the epoch schedule. For any $\delta\in(0,1)$, we define the following quantity
    \begin{align*}
    \Gamma_s:=&(\max\{4, \lambda\} + 1)r_{2^{s-2}}+ \sqrt{\cE^{\cF}_{\PP}(\tilde{g})} + 4\max\{4, \lambda\} \sqrt{\frac{2\log(4s^2\varphi_{\lambda}(r_{2^{s-2}})/\delta)}{2^{s-2}}}\\
    +& \frac{\log(4s^2\varphi_{\lambda}(r_{2^{s-2}})/\delta)}{r_{2^{s-2}} 2^{s-2}} +\frac{4}{\upsilon(f^*,\PP,\tilde{g})^{1/2}}\sqrt{\frac{\log(8s^2/\delta)}{2^{s-2}}}.
    \end{align*}
    We further define the exploration parameter in \cref{alg:IGW_pseudo_code} as
    \[
    \gamma_s:=\frac{\sqrt{K}}{2\Gamma_s},\ \forall s\ge 2.
    \]
    Then, we have that with probability at least $1-\delta$, 
    \begin{align*}
    \text{Reg}(T)\lesssim& \sqrt{K}\Big\lbrace\sum_{s=2}^{s(T)}r_{2^{s-2}}2^{s-1}+\sqrt{\cE^{\cF}_{\PP}(\tilde{g})}T+\sum_{s=2}^{s(T)}\sqrt{\log\rbr{\frac{4s^2\varphi_{\lambda}(r_{2^{s-2}})}{\delta}}+\log\rbr{\frac{8s^2}{\delta}}}2^{\frac{s}{2}}\\
    &+2\sum_{s=2}^{s(T)}\frac{\log\rbr{\frac{4s^2\varphi_{\lambda}(r_{2^{s-2}})}{\delta}}}{r_{2^{s-2}}}\Big\rbrace.
    \end{align*}
\end{theorem}
To better interpret the implications of \cref{thm:regret}, we provide an intuitive explanation of its components. The first term represents the regret arising from the statistical error in estimating the reward function—this reflects the inherent difficulty of function approximation within the chosen class. The second term, $\sqrt{\cE^{\cF}_{\PP}(\tilde{g})} \cdot T$, quantifies the regret attributable to imputing missing covariates using the pre-trained model $\tilde{g}$. This contribution is governed by the \emph{model elasticity}, which captures how sensitive the reward function is to errors in covariate imputation. The remaining two terms correspond to higher-order fluctuation effects that stem from reward function estimation, as characterized in the oracle inequality of \cref{thm:erm_oracle_ineq}.

For example, for parametric function families, we know that $r_n\asymp\frac{1}{N^{1/2}}$ \citep{rakhlin2022mathstat}. Therefore, by some algebra, \cref{alg:IGW_pseudo_code} yields a regret rate of $\Tilde{\cO}(T^{1/2})+\sqrt{\cE^{\cF}_{\PP}(\tilde{g})}T$. On the other hand, for non-parametric function families with order $0<p<1/2$, we know that $r_n\asymp\frac{1}{n^{p}}$, and thus applying $\cref{alg:IGW_pseudo_code}$ gives us a regret rate of $\widetilde{\cO}(T^{1-p}+\sqrt{\cE^{\cF}_{\PP}(\tilde{g})}T)$.

A notable non-desirable term in \cref{thm:regret} is $\sqrt{\cE^{\cF}_{\PP}(\tilde{g})} \cdot T$, which scales linearly with the time horizon $T$, resembling the regret behavior observed under model misspecification \citep{foster2020adapting, bogunovic2021misspecified}. In \cref{sec:model_calibration}, we introduce an effective approach to eliminate this term by leveraging model calibration techniques grounded in orthogonal statistical learning and doubly robust regression. This calibration procedure corrects for the bias introduced by imputation and restores favorable regret guarantees.

\section{Pre-Trained Model Calibration for Decision-Making}\label{sec:model_calibration}
In this section, we show that when the covariate missingness mechanism satisfies the \emph{Missing At Random} (MAR) assumption, it is possible to calibrate the pre-trained model \(\tilde{g}\) using the data collected during the learning process. More importantly, this calibration enables us to achieve a much better regret, which consists of sublinear terms plus a much tinier unavoidable linear term dependent on the random covariate noise that cannot be eliminated by model calibration.

We begin by introducing the structural assumption on the missingness mechanism—namely, the \emph{Missing At Random} (MAR) condition. Under MAR, the missingness indicator \(b \in \{0,1\}\) depends only on the observed context \(x\), and not on the unobserved covariate \(z^*\). This assumption allows the missingness probability to be modeled as a function \(e^*(x) = \PP(b = 1 \mid x)\), which we assume lies in a known function class \(\cT\). We further impose regularity and boundedness conditions on \(\cT\) to ensure the identifiability and learnability of the missingness mechanism.

In addition to the missingness model, we make structural assumptions on the covariate-generating process. Specifically, we assume that the true covariate function \(g^*\), as well as the pre-trained model \(\tilde{g}\), both belong to a common function class \(\cG\). The class \(\cG\) is assumed to be nonparametric and potentially infinite-dimensional, so as to accommodate complex models such as neural networks or kernel-based regressors. While \(\cG\) is allowed to be rich, we assume it admits mild complexity control via covering number bounds. These assumptions jointly enable us to characterize the quality of the imputed covariates and derive high-probability guarantees on the learned reward function. Mathematically, we summarize our assumptions as below in \cref{ass:calibration}.
\begin{assumption}\label{ass:calibration}
We assume that the following conditions about our decision-making problem are true:
    \begin{itemize}
        \item[(A)] We have access to a function class $\cG$ such that 
        \begin{itemize}
            \item[1:]  $\cG$ is $1$-uniformly bounded.
            \item[2:] the covariate $z^*$ could be written as $z^*=g^*(x)+\eta$, where $g^*\in\cG$ and $\eta$ is zero-mean stochastic noise bounded in $[-\tau,\tau]$ with second moment bound $\EE[\eta^2]^{1/2}\le\omega_0$.
            \item[3:] $\cG$ is a non-parametric function family, namely, for some number $d>0$,
            \[
            \log\cN(\epsilon,\cG,||\cdot||_{2})\lesssim \rbr{\frac{1}{\epsilon}}^d.
            \]
            \item[4:] We have prior knowledge about the pre-trained model $\tilde{g}$ such that $\tilde{g}$ $\in \cG$, and $||\tilde{g}-g^*||_{2}\le \delta_0$ for some known number $\delta_0$.
        \end{itemize}
        \item[(B)] We have access to a function class $\cT:\cX\mapsto[\epsilon_0,1]$ such that
         $$\exists e^*\in\cT, \text{such that}\ \PP(b=1|x)=e^*(x).$$ 
         \item[(C)] The reward model class $\cF$ satisfies that $\forall f\in \cF$, $x\in\cX,\ a\in\cA$,
        \[
        |f(x,z_1,a)-f(x,z_2,a)|\le L_{\cZ}|z_1-z_2|,
        \]
        i.e., $\cF$ is Lipschitz with respect to the covariate.
    \end{itemize}
\end{assumption}
We now provide a brief interpretation of \cref{ass:calibration}. In part (A), we assume that both the pre-trained model $\Tilde{g}$ and the unknown ground-truth model $g^*$ belong to a common function class $\cG$, which may be nonparametric and sufficiently rich to capture complex relationships. The number $d$ here controls the complexity of the function class $\cG$. In fact, one can show that if the fat-shattering dimension of $\cG$ at level $\epsilon$ satisfies $\text{fat}_{\epsilon}(\cG)\lesssim \frac{1}{\epsilon^d}$, then we have $\log\cN(\epsilon,\cG,||\cdot||_{2})\lesssim \rbr{\frac{1}{\epsilon}}^d$\citep{alon1997scale}. The final condition in (A) reflects a practical consideration: for the decision maker (DM) to rely on the pre-trained model $\Tilde{g}$ for imputing missing covariates, it must be reasonably accurate. Accordingly, we assume that the $L_2$ distance between $\tilde{g}$ and $g^*$ is bounded by a known constant $\delta_0$, thereby quantifying the permissible deviation between the pre-trained and true models. This assumption ensures that the imputed values used in decision-making are not arbitrarily far from the truth.

For the remaining components of \cref{ass:calibration}, part (B) ensures that the missingness mechanism $e^*$ is identifiable and can be estimated from the observed data. There are many ways of modeling the missing at random mechanism. For example, logit and probit models \citep{seaman2013meant,hsiao1996logit,heckman1976common}. Part (C) captures the key insight that imputation error—i.e., the discrepancy between the predicted and true covariate values—can be systematically translated into a bound on the reward estimation error, thus allowing us to control its impact on the decision-making process.

We now introduce our calibration procedure, denoted by $\Cal$. This method draws inspiration from recent advances in orthogonal statistical learning \citep{foster2023orthogonalstatisticallearning} and double machine learning \citep{chernozhukov2017double}, and consists of two main components.

First, we estimate the missingness mechanism function $e^*$ using the observed data, obtaining an estimate denoted by $\hat{e}$. This function plays the role of a \emph{nuisance parameter} in the spirit of orthogonal or debiased machine learning frameworks, and its accurate estimation is crucial for mitigating bias introduced by non-random missingness.

In the second step, we incorporate the estimate $\hat{e}$ into a weighted regression procedure over a centered function class $\cG_{\delta_0} = \cbr{ g \in \cG : | g - \tilde{g} |_2 \leq \delta_0 }$, which restricts attention to functions that are close to the pre-trained model $\tilde{g}$. This step yields a calibrated model $\hat{g}$ that accounts for both the imputation error and the estimated missing mechanism.

To avoid data endogeneity and ensure independence between estimation stages, we employ a standard cross-fitting technique. Specifically, we randomly partition the available i.i.d. dataset $\cD = \cbr{(x_i, b_i, \tilde{z}_i, a_i, r_i)}_{i=1}^{N}$ into two approximately equal-sized subsets, denoted by $\cD_1$ and $\cD_2$. The first subset $\cD_1$ is used to estimate the missingness model $\hat{e}$, and the second subset $\cD_2$ is used to perform the regression step with the plugged-in estimate $\hat{e}$ to obtain $\hat{g}$.
We now present the calibration procedure pseudo-code in \cref{alg:calibration_procedure}.

\begin{algorithm}[ht]
\caption{Pre-Trained Model Calibration $\Cal$}\label{alg:calibration_procedure}
\begin{algorithmic}
    \State \textbf{Require:} Function class $\cG$, $\cT$, dataset $\cD=\cbr{(x_i,b_i,\tilde{z}_i,a_i,r_i)}_{i=1}^{N}$, pre-trained model $\tilde{g}$ and $\delta_0$.
    \State Randomly split $\cD$ into two equal-sized datasets $$\cD_1=\cbr{(x_i,b_i,\tilde{z}_i,a_i,r_i)}_{i=1}^{N/2},\ \cD_2=\cbr{(x_i,b_i,\tilde{z}_i,a_i,r_i)}_{i=N/2+1}^{N}$$ without overlapping (assuming $N$ is even).
    \State \textbf{Define} the centered function class using the pre-trained model $$\cG_{\delta_0} = \cbr{ g \in \cG : | g - \tilde{g} |_2 \leq \delta_0 }.$$
    
    \State \textbf{For} the data in $\cD_1$, solve the following regression problem
    \[
    \hat{e}:=\argmin_{e\in\cT}\frac{2}{N}\sum_{i=1}^{N/2}(e(x_i)-b_i)^2.
    \]
    \State\textbf{For} the data in $\cD_2$, solve the following regression problem
    \[
    \hat{g}:=\argmin_{g\in\cG_{\delta_0}}\frac{2}{N}\sum_{i=N/2+1}^{N}\frac{b_i(b_iz_i^*-g(x_i))^2}{\hat{e}(x_i)}.
    \]
    \State \textbf{Output} $\hat{e}$ and $\hat{g}$.
\end{algorithmic}
\end{algorithm}
Now we are ready to give the statistical guarantee of the calibration procedure. For technical simplicity and without loss of generality, we assume that $\cT$ and $\cG$ are convex classes.

\begin{theorem}\label{thm:calibration_guarantee}
    Denote $s_N$ as the critical radius of the function class $\cT$ with dataset cardinality $N$ and $q_{\delta_0,N}$ as the critical radius of the centered function class $\cG_{\delta_0}$ with dataset cardinality $N$. Then given any i.i.d. dataset $\cD$ such that $|\cD|=N$, we have that for any $\delta\in(0,1)$, with probability at least $1-2\delta$,
    \[
    ||\hat{g}-g^*||_2^2\lesssim \frac{1}{\epsilon_0^2}\rbr{\delta_0^{\frac{2d}{d+2}}N^{\frac{-2}{d+2}}+\frac{(1+\tau)^2\log(\varphi_{1+\tau}(q_{\delta_0,N/2})/\delta)}{N}+s_{N/2}^2+\frac{\log(\varphi_2(s_{N/2})/\delta)}{N}},
    \]
    where $s_N$ is the critical radius of the function class $\cT$.
\end{theorem}
  
If $\cT$ is a parametric function family, then we know that $s_N\asymp\frac{1}{\sqrt{N}}$ \citep{wainwright2019high}. Therefore, the dominant term in \cref{thm:calibration_guarantee} is $\delta_0^{\frac{2d}{d+2}}N^{\frac{-2}{d+2}}$, which highlights the interaction between the quality of the pre-trained model (captured by $\delta_0$)and the sample size $N$. As the pre-trained model becomes more accurate—i.e., as 
$\delta_0\rightarrow0$—the overall estimation error decreases polynomially. This demonstrates that better prior models can significantly reduce the effective complexity of the localized function class, enabling faster convergence rates for non-parametric function families.

Now we present our decision-making algorithmic framework with calibration in \cref{alg:IGW_calibration} and denote it as ``$\mathsf{PRIMO-Cal}$''.
\begin{algorithm}[ht]
\caption{$\mathsf{PRIMO-Cal}$}\label{alg:IGW_calibration}
\begin{algorithmic}
    \State \textbf{Require:} Context space $\cX$; covariate space $\cZ$; action space $\cA$; function class $\cF$; pre-trained model $\tilde{g}$; tuning parameters $\cbr{\gamma_s}_{s=1}^{\infty}$
    \State Initialize epoch schedule $0 = \beta_0 < \beta_1 < \beta_2 < \cdots$
    \For{epoch $s = 1, 2, \cdots$}
        \State Construct dataset $\cD_{s-1} = \cbr{(x_t,b_t, \tilde{z}_t, a_t, r_t)}_{t = \beta_{s-2}+1}^{\beta_{s-1}}$.
        \State Fit reward function $\hat{f}_s = \Alg(\cD_{s-1})$.
        \State Carry out the model calibration procedure $\Cal(\cD_{s-1})$ to get $\hat{e}_s$, $\hat{g}_s$. 
        \For{round $t = \beta_{s-1} + 1, \dots, \beta_s$}
            \State Observe context $x_t$,
            \State If $z_t^*$ is observed, set $\tilde{z}_t = z_t^*$; otherwise set $\tilde{z}_t = \hat{g}_{s}(x_t)$.
            \State Define policy $\texttt{IGW}_{\gamma_s, \hat{f}_s}$:
            \begin{equation*}
                p_t(a) =
                \begin{cases}
                    \frac{1}{K + \gamma_s \left( \hat{f}_s(x_t, \tilde{z}_t, \hat{a}^s_{x_t,\tilde{z}_t}) - \hat{f}_s(x_t, \tilde{z}_t, a) \right)} & \text{if } a \neq \hat{a}^s_{x_t,\tilde{z}_t} \\
                    1 - \sum_{a \neq \hat{a}^s_{x_t,\tilde{z}_t}} p_t(a) & \text{if } a = \hat{a}^s_{x_t,\tilde{z}_t}.
                \end{cases}
            \end{equation*}
            \State Sample action $a_t \sim p_t$ and receive reward $r_t$.
        \EndFor
    \EndFor
\end{algorithmic}
\end{algorithm}

With \cref{thm:calibration_guarantee} and \cref{alg:IGW_calibration}, we are able to improve the estimation error of the reward function, and thus provide a better regret guarantee for \cref{alg:IGW_calibration}, which is presented in \cref{thm:regret_calibration}.
\begin{theorem}\label{thm:regret_calibration}
    Under the missing at random mechanism, for \cref{alg:IGW_calibration}, if we set $\beta_s=2^s$ in the epoch schedule. For any $\delta>0$, we define the following quantities:
    \begin{align*}
        \Upsilon_s:= \frac{L_{\cZ}}{\epsilon_0}\rbr{\delta_0^{\frac{d}{d+1}}2^{\frac{-2(s-3)}{d+2}}+\frac{(1+\tau)\sqrt{2s^2\log(\varphi_{1+\tau}(q_{\delta_0,2^{s/2-3/2}})/\delta)}}{2^{\frac{s-3}{2}}}+s_{2^{s-3}}+\frac{\log(2s^2\varphi_2(s_{2^{s-3}})/\delta)}{2^{\frac{s-3}{2}}}}+L_{\cZ}\omega_0,
    \end{align*}
    \begin{align*}
    \Gamma_s:=&(\max\{4, \lambda\} + 1)r_{2^{s-2}}+ \Upsilon_s + 4\max\{4, \lambda\} \sqrt{\frac{2\log(4s^2\varphi_{\lambda}(r_{2^{s-2}})/\delta)}{2^{s-2}}}\\
    +&\frac{\log(4s^2\varphi_{\lambda}(r_{2^{s-2}})/\delta)}{r_{2^{s-2}} 2^{s-2}} +\frac{4}{\upsilon(f^*,\PP,\tilde{g})^{1/2}}\sqrt{\frac{\log(8s^2/\delta)}{2^{s-2}}}.
    \end{align*}
     We set the exploration parameter as $\gamma_s=\frac{\sqrt{K}}{2\Gamma_s}$. Then we have that with probability at least $1-3\delta$, we have that
\begin{align*}
    \text{Reg}(T)\lesssim& \sqrt{K}\rbr{\sum_{s=2}^{s(T)}r_{2^{s-2}}2^{s-1}+\frac{L_{\cZ}}{\epsilon_0}\delta_0^{\frac{d}{d+2}}\sum_{s=2}^{s(T)}2^{\frac{-2(s-3)}{d+2}+s-1}+\frac{L_{\cZ}}{\epsilon_0}\sum_{s=2}^{s(T)}s_{2^{s-3}}2^{s-1}}\\
    +&\sqrt{K}\rbr{\sum_{s=2}^{s(T)}\sqrt{\log(4s^2\varphi_{\lambda}(r_{2^{s-2}})/\delta)+\log(8s^2/\delta)}2^{\frac{s}{2}}+\sum_{s=2}^{s(T)}\log(\varphi_{\lambda}(4s^2r_{2^{s-2}})/\delta)\frac{2}{r_{2^{s-2}}}}\\
    +&\frac{\sqrt{K}L_{\cZ}}{\epsilon_0}\rbr{\sqrt{(1+\tau)^2\log(s^2\varphi_{1+\tau}(q_{\delta_0,2^{s-3}})/\delta)+\log(s^2\varphi_2(s_{2^{s-3}})/\delta)}}2^{s/2}\\
    +&L_{\cZ}\omega_0 T.
\end{align*}
\end{theorem}

In \cref{thm:regret_calibration}, we highlight that the regret bound consists of four distinct types of terms.

The \textbf{first line} contains the term
\[
\sum_{s=2}^{s(T)} r_{2^{s-2}} \cdot 2^{s-1},
\]
which corresponds to the regret incurred due to the oracle risk in estimating the reward function.

The \textbf{next group of terms},
\[
\frac{L_{\cZ}}{\epsilon_0} \delta_0^{\frac{d}{d+2}} \sum_{s=2}^{s(T)} 2^{\frac{(d-2)s}{d+2}} \quad \text{and} \quad \frac{L_{\cZ}}{\epsilon_0} \sum_{s=2}^{s(T)} s_{2^{s-3}} \cdot 2^{s-1},
\]
represents additional regret arising from \emph{model calibration errors}. This component has two sources:
\begin{itemize}
    \item The term
    \(
    \frac{L_{\cZ}}{\epsilon_0} \sum_{s=2}^{s(T)} s_{2^{s-3}} \cdot 2^{s-1}
    \)
    is due to the approximation error in the missingness mechanism estimation, i.e., the error when estimating $e^*$.
    \item The term
    \(
    \frac{L_{\cZ}}{\epsilon_0} \delta_0^{\frac{d}{d+2}} \sum_{s=2}^{s(T)} 2^{\frac{(d-2)s}{d+2}}
    \)
    stems from \textit{calibrating the pre-trained model} within the localized function class $\cG_{\delta_0}$. Notably, this latter term vanishes as the prior model $\tilde{g}$ becomes increasingly accurate, i.e., as $\delta_0 \to 0$.
\end{itemize}

Terms in the \textbf{second line} account for regret caused by \emph{higher-order fluctuation errors} in reward function estimation, which are residual statistical complexities beyond the leading oracle rate.

The \textbf{third line} of the bound captures the \emph{higher-order fluctuation terms} that arise specifically from calibrating the pre-trained model. These reflect the additional complexity introduced when combining plug-in imputation with adaptive estimation of the nuisance components.

\textbf{Finally}, the term $\omega_0 L_{\cZ}T$ arises due to the inherent noise in the covariate imputation process. Specifically, for any estimator function $g$, we have $$\sqrt{\EE[(z^*-g(x))^2]}=\sqrt{\EE[((g^*-g)+\eta)^2]}\lesssim\sqrt{\EE[\eta^2]}+||g^*-g||_2\le \omega_0+||g^*-g||_2$$ 
where $\eta$ denotes the stochastic noise independent of $g$. Even with perfect calibration of the pre-trained model, the irreducible noise term $\sqrt{\EE[\eta^2]}$ persists and contributes to the regret. Hence, this component cannot be eliminated through pre-trained model calibration.

As an illustrative example, consider the case where both the function classes $\cF$ and $\cT$ are parametric. In this setting, the statistical complexities satisfy $s_n \asymp n^{-1/2}$ and $r_n \asymp n^{-1/2}$. Substituting these rates into the regret bound yields
\[
\text{Reg}(T)\le\widetilde{\cO}\rbr{T^{1/2}+\delta_0^{\frac{d}{d+2}}T^{\frac{d-2}{d+2}}+\omega_0 T}.
\]
In this way, we eliminate the linear term in \cref{thm:regret} and achieve a much better regret bound that depends on the critical radii of the function classes $\cF,\ \cG$ and $\cT$ plus an inevitable much smaller linear term. As we can see, the accuracy of the pre-trained model becomes crucial to the overall performance, as it directly impacts the regret stemming from model calibration error. More precisely, the cumulative regret of the calibrated algorithm in \cref{alg:IGW_calibration} is dictated by the statistical complexity of the most challenging function class involved in the learning process.


\section{Discussion and Future Directions}

In this paper, we studied sequential decision-making with missing covariates, where a pre-trained model is available to impute the unobserved information. We introduced a novel notion of \emph{model elasticity} to quantify how the sensitivity of the reward function to covariate imputation affects regret. This concept allows us to unify the analysis across different missingness mechanisms, including the challenging missing not at random (MNAR) setting. Furthermore, under the more benign missing at random (MAR) assumption, we proposed a calibration procedure that leverages orthogonal statistical learning and doubly robust regression to improve the performance of the pre-trained model and improve the regret bound. Our results demonstrate that even in the presence of unobserved covariates, a carefully calibrated pre-trained model can significantly reduce regret in online learning. Moreover, we formally establish that the more accurate the pre-trained model is—as a form of prior knowledge—the smaller the resulting regret, aligning with intuitive expectations.

There are several promising directions for future research. We leave the exploration of these directions to the future. First, our calibration procedure currently relies on the model-based methods for estimating the missingness mechanism under MAR. An important extension would be to incorporate modeling and learning the missing not at random (MNAR) mechanism while still maintaining provable guarantees. Second, our notion of model elasticity could be further explored in other sequential decision-making settings, such as reinforcement learning, causal inference, and active learning under partial observability. Third, while this work focuses on a single pre-trained model, it would be valuable to extend the framework to settings involving multiple imperfect pre-trained models, and to develop principled algorithms that adaptively select or ensemble information from them over time. Moreover, evaluating our algorithm in a broader range of real-world applications would offer practical insights and further validate its effectiveness. Finally, we believe that pre-trained AI models have the potential to assist a wide range of downstream decision-making tasks. Analyzing the theoretical properties of them presents a promising direction, particularly in the current era of generative AI.

\newpage
\bibliographystyle{plainnat}
\bibliography{sections/refs}
\newpage
\appendix
\section{Useful Mathematical Tools}
\begin{theorem}\citep{klein2005concentration}\label{thm: sup_empirical_process}
    Consider a countable, $\sigma$-uniformly bounded function class $\cG$ such that $\EE[g(X)]=0$ for all $g$. Then for any $\tau>0$, we have with probability at least $1-\delta$,
    \[
    \sup_{g\in\cG}\sqrt{\frac{1}{N}\sum_{i=1}^{N}g(X_i)^2}=||\PP_{N}||_{\cG}\le (1+\tau)\EE[||\PP_{N}||_{\cG}]+\sqrt{\sigma^2(\cG)}\cdot\sqrt{\frac{2\log(1/\delta)}{N}}+(3+\frac{1}{\tau})\frac{\sigma\log(1/\delta)}{N}.
    \]
\end{theorem}

Next, we introduce the peeling argument in \citet{xia2024predictionaidedsurrogatetraining}.
\begin{lemma}\citep{xia2024predictionaidedsurrogatetraining}\label{lemma:peeling}
Given a function class $\cF$, with norm $||\cdot||$ and some empirical process $\cbr{V_n(f):f\in\cF}$, consider a $r$-localized supremum of the form $Z_N(r):=\sup_{f\in\cF,||f||\le r}V_N(f)$. There is a function $Q(r,t)$ increasing in the first argument $r$ and $Q(2r,t)\le 2Q(r,t)$ for all $r\ge s$, where $s$ is some scalar. If we have \[
\PP(Z_N(r)\ge Q(r,t))\le e^{-t},\ \forall r\ge s, 
\]
then, for any random variable $U$ taking values in $[s,b]$, we have
\[
\PP(Z_N(U)\ge 2Q(U,t))\le \ceil{\log_2(\frac{2b}{s})}e^{-t},\ \forall r\ge s,
\]
\end{lemma}
\begin{lemma}[Talagrand Inequality \citep{gine2021mathematical}]\label{lemma:Talagrand_ineq}
    Let $\cF$ be a separable function class with absolute values bounded by $1$, let $X_i,\ i\in\NN$ be independent, identically distributed random variables with common probability law $\PP$ and let $\varepsilon_i$, $i\in\NN$ be a Rademacher sequence from the sequence $\cbr{X_i}$. Then, for all $n\in\NN$ and $x\ge 0$,
    \[
    \PP\cbr{\bignorm{\frac{1}{n}\sum_{i=1}^{n}(f(X_i)-\PP f)}_{\cF}\ge2\bignorm{\frac{1}{n}\sum_{i=1}^{n}f(X_i)}_{\cF}+3\sqrt{\frac{2x}{n}}}\le 2e^{-x}.
    \]
\end{lemma}
\begin{lemma}[Ledoux-Talagrand Lemma]\label{lemma:Ledoux_Talagrand}
    Let $\{\varepsilon_i\}_{i=1}^n$ be a sequence of independent Rademacher random variables (i.e., symmetric Bernoulli random variables), and let $\phi_i : \mathbb{R} \to \mathbb{R}$ be contraction mappings, i.e., $L$-Lipschitz functions such that $\phi_i(0) = 0$. Let $F : [0, \infty) \to \mathbb{R}$ be a nonnegative, nondecreasing convex function. Then for any bounded subset $T \subset \mathbb{R}^n$, we have:
\[
\mathbb{E} \left[ F\left( \frac{1}{2} \sup_{t \in T} \left| \sum_{i=1}^n \varepsilon_i \phi_i(t_i) \right| \right) \right] 
\le 
L\mathbb{E} \left[ F\left( \sup_{t \in T} \left| \sum_{i=1}^n \varepsilon_i t_i \right| \right) \right].
\]
\end{lemma}
\begin{lemma}[McDiarmid's Inequality]\label{lemma:micdiarmid}
    Let \( X_1, X_2, \dots, X_n \) be independent random variables taking values in some sets \( \mathcal{X}_1, \mathcal{X}_2, \dots, \mathcal{X}_n \), respectively.  
Suppose that a function \( f : \mathcal{X}_1 \times \cdots \times \mathcal{X}_n \to \mathbb{R} \) satisfies the bounded difference property:  
\[
\sup_{x_1, \dots, x_n, x_i'} \left| f(x_1, \dots, x_i, \dots, x_n) - f(x_1, \dots, x_i', \dots, x_n) \right| \le c_i, \quad \text{for all } i = 1, \dots, n.
\]
Then, for all \( \varepsilon > 0 \),
\[
\mathbb{P}\left( f(X_1, \dots, X_n) - \mathbb{E}[f(X_1, \dots, X_n)] \ge \varepsilon \right) \le \exp\left( -\frac{2\varepsilon^2}{\sum_{i=1}^n c_i^2} \right),
\]
and similarly,
\[
\mathbb{P}\left( \left| f(X_1, \dots, X_n) - \mathbb{E}[f(X_1, \dots, X_n)] \right| \ge \varepsilon \right) \le 2 \exp\left( -\frac{2\varepsilon^2}{\sum_{i=1}^n c_i^2} \right).
\]
\[
\left| f(X_1, \dots, X_n) - \mathbb{E}[f(X_1, \dots, X_n)] \right| \le \sqrt{\frac{1}{2} \sum_{i=1}^n c_i^2 \cdot \log\left( \frac{2}{\delta} \right)}
\]
\end{lemma}
\begin{lemma}[Hoeffding Inequality for Lipschitz Transformations]\label{lemma:Hoeffding_Lipschitz}
Let $X_1, \dots, X_n$ be independent random variables with $X_i \in [a, b]$ almost surely for all $i \in \{1, \dots, n\}$, and define the empirical mean $\bar{X}_n := \frac{1}{n} \sum_{i=1}^n X_i$. Suppose $f : \mathbb{R} \to \mathbb{R}$ is an $L$-Lipschitz function, i.e.,
\[
|f(x) - f(y)| \le L |x - y|, \quad \forall x,y \in \mathbb{R}.
\]
Then, for any $\varepsilon > 0$,
\[
\mathbb{P}\left( \left| f(\bar{X}_n) - f\left(\mathbb{E}[X_1]\right) \right| \ge \varepsilon \right) 
\le 2 \exp\left( -\frac{2n\varepsilon^2}{L^2(b - a)^2} \right).
\]
Then, for any $\delta \in (0,1)$, with probability at least $1 - \delta$, we have
\[
\left| f(\bar{X}_n) - f\left(\mathbb{E}[X_1]\right) \right| 
\le L(b - a) \sqrt{\frac{\log(2/\delta)}{2n}}.
\]
\end{lemma}
\section{Proofs in \cref{sec:algorithm}}
First, we analyze the estimation error of \cref{thm:erm_oracle_ineq} in this section. First, we have
\begin{align*}
    ||\hat{f}-f^*||_2^2=||\hat{f}-f^*||_2^2-||\hat{f}-f^*||_N^2+||\hat{f}-f^*||_N^2.
\end{align*}
Since the ERM minimizes the empirical risk, we have
\[
\frac{1}{N}\sum_{i=1}^{N}\rbr{f^*(x_i,z_i^*,a_i)+\xi_i-\hat{f}(x_i,\Tilde{z}_i,a_i)}^2\le\frac{1}{N}\sum_{i=1}^{N}\rbr{f^*(x_i,z_i^*,a_i)+\xi_i-f^*(x_i,\Tilde{z}_i,a_i)}^2.
\]
By some algebra,
\begin{align*}
    &\frac{1}{N}\sum_{i=1}^{N}\rbr{f^*(x_i,z_i^*,a_i)+\xi_i-f^*(x_i,\Tilde{z}_i,a_i)-(\hat{f}(x_i,\Tilde{z}_i,a_i)-f^*(x_i,\Tilde{z}_i,a_i))}^2\\
    \le&\frac{1}{N}\sum_{i=1}^{N}\rbr{f^*(x_i,z_i^*,a_i)+\xi_i-f^*(x_i,\Tilde{z}_i,a_i)}^2.
\end{align*}
Therefore,
\begin{align*}
    &\frac{1}{N}\sum_{i=1}^{N}\rbr{f^*(x_i,\Tilde{z}_i,a_i)-\hat{f}(x_i,\Tilde{z}_i,a_i)}^2\\
    \le&\frac{2}{N}\sum_{i=1}^{N}\rbr{f^*(x_i,z_i^*,a_i)+\xi_i-f^*(x_i,\Tilde{z}_i,a_i)}\rbr{\hat{f}(x_i,\Tilde{z}_i,a_i)-f^*(x_i,\Tilde{z}_i,a_i)}\\
    =&\frac{2}{N}\sum_{i=1}^{N}\xi_i\rbr{\hat{f}(x_i,\Tilde{z}_i,a_i)-f^*(x_i,\Tilde{z}_i,a_i)}+\frac{2}{N}\sum_{i=1}^{N}\rbr{f^*(x_i,z_i^*,a_i)-f^*(x_i,\Tilde{z}_i,a_i)}\rbr{\hat{f}(x_i,\Tilde{z}_i,a_i)-f^*(x_i,\Tilde{z}_i,a_i)}.
\end{align*}
For the term $\sum_{i=1}^{N}\rbr{f^*(x_i,z_i^*,a_i)-f^*(x_i,\Tilde{z}_i,a_i)}\rbr{\hat{f}(x_i,\Tilde{z}_i,a_i)-f^*(x_i,\Tilde{z}_i,a_i)}$, we have,
\begin{align*}
&\sum_{i=1}^{N}\rbr{f^*(x_i,z_i^*,a_i)-f^*(x_i,\Tilde{z}_i,a_i)}\rbr{\hat{f}(x_i,\Tilde{z}_i,a_i)-f^*(x_i,\Tilde{z}_i,a_i)}\\
=&\frac{2}{N}\sum_{i=m+1}^{m+n}\rbr{f^*(x_i,z_i^*,a_i)-f^*(x_i,\Tilde{z}_i,a_i)}\rbr{\hat{f}(x_i,\Tilde{z}_i,a_i)-f^*(x_i,\Tilde{z}_i,a_i)}.
\end{align*}
Combining all these together, we have,
\begin{align*}
    ||\hat{f}-f^*||_2^2\le& \frac{2}{N}\sum_{i=1}^{N}\xi_i\rbr{\hat{f}(x_i,\Tilde{z}_i,a_i)-f^*(x_i,\Tilde{z}_i,a_i)}+||\hat{f}-f^*||_2^2-||\hat{f}-f^*||_N^2\\
    +& \frac{2}{N}\sum_{i=m+1}^{m+n}\rbr{f^*(x_i,z_i^*,a_i)-f^*(x_i,\Tilde{z}_i,a_i)}\rbr{\hat{f}(x_i,\Tilde{z}_i,a_i)-f^*(x_i,\Tilde{z}_i,a_i)}
\end{align*}
Now we are ready to analyze the error term by term. As a reminder, recall $\varphi_{b}(x)=\log_2(4b/x)$ as defined in \cref{sec:algorithm}.
\begin{lemma}\label{lemma:bound_K_1}
    Define $K_1=||\hat{f}-f^*||_2^2-||\hat{f}-f^*||_N^2$, then conditioned on $||\hat{f}-f^*||_2\ge r_N$, with probability at least $1-\delta$, 
    \[
    K_1\le 2r_N||\hat{f}-f^*||_2+8\sqrt{\frac{2\log(\varphi_{\lambda}r_N)/\delta)}{N}}||\hat{f}-f^*||_2+\frac{16\log(\varphi_{\lambda}r_N)/\delta)}{N}.
    \]
\end{lemma}
\begin{proof}[Proof of \cref{lemma:bound_K_1}]
We define the function $H_f(x,z,a)=(f(x,z,a)-f^*(x,z,a))^2$ for all $f\in\cF$. Define
\[
Z_N(t):=\sup_{f\in\cH(t)}\cbr{\PP_N(H_f)-\PP(H_f)},
\]
where $\cH(t):=\cbr{H_f:f\in\cF,\ \text{s.t.} ||f-f^*||_2\le t}$. By construction, we have
$$K_1\le Z_N(||\hat{f}-f^*||_2).$$
    We then set $\tau=1$ and apply \cref{thm: sup_empirical_process}. Functions in $\cH(t)$ are uniformly bounded by $2$. to get with probability at least $1-\delta$,
    \[
    Z_N(t)\le 2\EE[Z_N(t)]+\sqrt{\sigma^2(\cH(t))}\sqrt{\frac{2\log(1/\delta)}{N}}+\frac{16\log(1/\delta)}{N}.
    \]
    On one hand, we have
    \[
    \sigma^2(\cH(t))=\sup_{||f-f^*||_2\le t}\PP_{\cD}\rbr{(f(x,z,a)-f^*(x,z,a))^4}\le 4\PP_{\cD}\rbr{(f(x,z,a)-f^*(x,z,a))^2}\le 4t^2.
    \]
    \[
    \EE[Z_N(t)]=\EE\sbr{\sup_{||f^*-f||_2\le t}\abr{||f^*-f||_N^2-||f-f^*||_2^2}}\le 2\EE\sbr{\sup_{||f^*-f||_2\le t}\abr{\frac{1}{N}\sum_{i=1}^{N}\varepsilon_i\rbr{f^*(x_i,\Tilde{z}_i,a_i)-f(x_i,\Tilde{z}_i,a_i)}^2}}.
    \]
    By \cref{lemma:Ledoux_Talagrand}, the function $\phi_i(t)=\phi(t)=t^2$ is $4$-Lipschitz on the interval $[0,2]$, we have
    \begin{align*}\hspace{-0.6cm}
    &2\EE\sbr{\sup_{||f^*-f||_2\le t}\abr{\frac{1}{N}\sum_{i=1}^{N}\varepsilon_i\rbr{f^*(x_i,\Tilde{z}_i,a_i)-f(x_i,\Tilde{z}_i,a_i)}^2}}\\
    \le& 8\EE\sbr{\sup_{||f^*-f||_2\le t}\abr{\frac{1}{N}\sum_{i=1}^{N}\varepsilon_i\rbr{f^*(x_i,\Tilde{z}_i,a_i)-f(x_i,\Tilde{z}_i,a_i)}}}\\
    =&8\cR_N(t,\cF).
    \end{align*}
    By definition, we have $\cR_N(t,\cF)\le\frac{t}{16}r_N$.
    Pluging all these in, we get with probability at least $1-\delta$,
    \[
    Z_N(t)\le r_Nt+4\sqrt{\frac{2\log(1/\delta)}{N}}t+\frac{16\log(1/\delta)}{N}.
    \]
Setting $\delta=e^{-\eta}$ and $\eta=\log(1/\delta)$, this is equivalent to
\[
\PP(Z_N(t)\ge r_Nt+4\sqrt{\frac{2\eta}{N}}t+\frac{8\eta}{N})\le e^{-\eta}.
\]
Define the random variable $U=||\hat{f}-f^*||_2$, and apply the peeling argument in \cref{lemma:peeling}, we shall get
\[
\PP\rbr{Z_N(U)\ge 2(r_NU+4\sqrt{\frac{2\eta}{N}}U+\frac{8\eta}{N})}\le \ceil{\log_2\frac{4}{r_N}}e^{-\eta}.
\]
Therefore, with probability $1-\delta$,
\[
Z_N(||\hat{f}-f^*||_2)\le 2r_N||\hat{f}-f^*||_2+8\sqrt{\frac{2\log(\varphi_2(r_N)/\delta)}{N}}||\hat{f}-f^*||_2+32\frac{\log(\varphi_2(r_N)/\delta)}{N}.
\]
\end{proof}

Now we bound the second term $K_2:=\frac{2}{N}\sum_{i=1}^{N}\xi_i\rbr{\hat{f}(x_i,\Tilde{z}_i,a_i)-f^*(x_i,\Tilde{z}_i,a_i)}$. 
\begin{lemma}\label{lemma:bound K_2}
    Condition on $||\hat{f}-f^*||_2\ge r_N$, we have that with probability at least $1-\delta$,
    \[
    K_2\le \lambda r_N||\hat{f}-f^*||_2+4\lambda\sqrt{\frac{2\log(\varphi_2(r_N)/\delta)}{N}}||\hat{f}-f^*||_2+\frac{32\lambda\log(\varphi_2(r_N)/
    \delta)}{N}.
    \]
\end{lemma}
\begin{proof}[Proof of \cref{lemma:bound K_2}]
    We define the following empirical process
    $$V_N(t):=\sup_{||\hat{f}-f^*||_2\le t}\abr{\frac{1}{N}\sum_{i=1}^{N}\xi_i\rbr{f^*(x_i,\Tilde{z}_i,a_i)-f(x_i,\Tilde{z}_i,a_i)}}.$$
    Define the function class $\cH(t):=\cbr{\xi(f-f^*):f\in\cF, ||f-f^*||_2\le t, |\xi|\le \lambda, \xi\in\QQ}$. By construction,
     \[
     K_2\le 2V_N(||\hat{f}-f^*||_2).
     \]
     Applying \cref{thm: sup_empirical_process} and setting $\tau=1$, we have with probability at least $1-\delta$,
    \[
    V_N(t)\le 2\EE[V_N(t)]+\sqrt{\sigma^2(\cH(t))}\sqrt{\frac{2\log(1/\delta)}{N}}+\frac{8\lambda\log(1/\delta)}{N}.
    \]
    We have
    \[
    \sigma^2(\cH(t))=\sup_{||f-f^*||_2\le t,|\xi|\le \lambda}\text{Var}(\xi(f-f^*))\le \lambda^2t^2.
    \]
    \[
    \EE[V_N(t)]=\EE\sbr{\sup_{||f-f^*||_2\le t}\abr{\frac{1}{N}\sum_{i=1}^{N}\varepsilon_i\xi_i\rbr{f(x_i,\Tilde{z}_i,a_i)-f^*(x_i,\Tilde{z}_i,a_i)}}}.
    \]
    Conditioned on $\xi_t$ and setting $\phi_i(t)=\frac{\xi_i}{\lambda}t$, we apply \cref{lemma:Ledoux_Talagrand} to get
    \begin{align*}\hspace{-0.9cm}
        &\EE\sbr{\sup_{||f-f^*||_2\le t}\abr{\frac{1}{N}\sum_{i=1}^{N}\varepsilon_i\xi_i\rbr{f(x_i,\Tilde{z}_i,a_i)-f^*(x_i,\Tilde{z}_i,a_i)}}}\\
        \le& 2\lambda\EE\sbr{\sup_{||f-f^*||_2\le t}\abr{\frac{1}{N}\sum_{i=1}^{N}\varepsilon_i\rbr{f(x_i,\Tilde{z}_i,a_i)-f^*(x_i,\Tilde{z}_i,a_i)}}}\\
        =&2\lambda\cR_N(t,\cF).
    \end{align*}
    By definition of the critical radius, for $t\ge r_N$, we have $\cR_{N}(t,\cF)\le \frac{r_N}{16}$. Then with probability at least $1-\delta$,
    \[
    V_N(t)\le \frac{1}{4}r_Nt+\lambda\sqrt{\frac{2\log(1/\delta)}{N}}t+\frac{8\lambda\log(1/\delta)}{N}.
    \]
    By the peeling \cref{lemma:peeling}, we define $U=||f^*-\hat{f}||_2$ and plug it in to get
    \[
    V_N(||\hat{f}-f^*||_2)\le \frac{1}{2}\lambda r_N||\hat{f}-f^*||_2+2\lambda\sqrt{\frac{2\log(\varphi_2(r_N)/\delta)}{N}}||\hat{f}-f^*||_2+\frac{16\lambda\log(\varphi_2(r_N)/
    \delta)}{N}.
    \]
\end{proof}
Finally, we bound the third term $K_3:=\frac{2}{N}\sum_{i=m+1}^{m+n}\rbr{f^*(x_i,z_i^*,a_i)-f^*(x_i,\Tilde{z}_i,a_i)}\rbr{\hat{f}(x_i,\Tilde{z}_i,a_i)-f^*(x_i,\Tilde{z}_i,a_i)}$. 

\begin{lemma}\label{lemma:bound_K_3}
    Conditioned on $||f^*-\hat{f}||_2\ge r_N$, with probability at least $1-2\delta$, we have
    \[
    \frac{K_3}{2}\le \rbr{\sqrt{\cE_{\PP}^{\cF}(\tilde{g})}+\frac{4}{\upsilon(f^*,\PP,\tilde{g})^{1/2}}\sqrt{\frac{\log(2/\delta)}{N}}}\rbr{||\hat{f}-f^*||_2+2r_N+8\sqrt{\frac{2\log(\varphi_{\lambda}r_N)/\delta)}{N}}+32\frac{\log(\varphi_{\lambda}r_N)/\delta)}{r_N N}}.
    \]
\end{lemma}
\begin{proof}[Proof of \cref{lemma:bound_K_3}]
    First, by the Cauchy-Schwarz inequality, we have
    \[
    K_3\le \frac{2n}{N}\sqrt{\frac{1}{n}\sum_{i=m+1}^{m+n}\rbr{f^*(x_i,z_i^*,a_i)-f^*(x_i,\Tilde{z}_i,a_i)}^2}\sqrt{\frac{1}{n}\sum_{i=m+1}^{m+n}\rbr{\hat{f}(x_i,\Tilde{z}_i,a_i)-f^*(x_i,\Tilde{z}_i,a_i)}^2}.
    \]
    By some algebra, we have
    \begin{align*}
        &\frac{2n}{N}\sqrt{\frac{1}{n}\sum_{i=m+1}^{m+n}\rbr{f^*(x_i,z_i^*,a_i)-f^*(x_i,\Tilde{z}_i,a_i)}^2}\sqrt{\frac{1}{n}\sum_{i=m+1}^{m+n}\rbr{\hat{f}(x_i,\Tilde{z}_i,a_i)-f^*(x_i,\Tilde{z}_i,a_i)}^2}\\
        \le&2\frac{\sqrt{n}}{\sqrt{N}}\sqrt{\frac{1}{n}\sum_{i=m+1}^{m+n}\rbr{f^*(x_i,z_i^*,a_i)-f^*(x_i,\Tilde{z}_i,a_i)}^2}\sqrt{\frac{1}{N}\sum_{i=1}^{N}\rbr{\hat{f}(x_i,\tilde{z}_i,a_i)-f^*(x_i,\Tilde{z}_i,a_i)}^2}.
    \end{align*}
Define the empirical process $Z_N(t):=\sup_{||f-f^*||_2\le t}\abr{||f-f^*||_N^2-||f-f^*||_2^2}$. We define $H_f(x,z,a)$ as $(f(x,z,a)-f^*(x,z,a))^2$ and $\cH(t):=\cbr{H_f:f\in\cF,||f-f^*||_2\le t}$. We then apply \cref{thm: sup_empirical_process} to get
\[
Z_N(t)\le 2\EE[Z_N(t)]+\sqrt{\sigma^2(\cH(t))}\sqrt{\frac{2\log(1/\delta)}{N}}+\frac{16\log(1/\delta)}{N}.
\]
Similar to \cref{lemma:bound_K_1}, we have
\[
\EE[Z_N(t)]\le 2\EE\sbr{\sup_{||f^*-f||_2\le t}\abr{\frac{1}{N}\sum_{i=1}^{N}\varepsilon_i\rbr{f^*(x_i,\Tilde{z}_i,a_i)-f(x_i,\Tilde{z}_i,a_i)}^2}}\le 8\cR_N(t,\cF)\le \frac{r_Nt}{2},
\]
and
\[
\sigma^2(\cH(t))\le 4t^2.
\]
Applying the peeling argument again, we have that with probability $1-\delta$,
\[
Z_N(||\hat{f}-f^*||_2^2)\le 2r_N||\hat{f}-f^*||_2+8\sqrt{\frac{2\log(\varphi_{\lambda}r_N)/\delta)}{N}}||\hat{f}-f^*||_2+32\frac{\log(\varphi_{\lambda}r_N)/\delta)}{N}.
\]
By construction, $\abr{||\hat{f}-f^*||_N^2-||\hat{f}-f^*||_2^2}\le Z_N(||\hat{f}-f^*||_2^2)$. Moreover, we have
\begin{align*}
    \abr{||\hat{f}-f^*||_2-||\hat{f}-f^*||_N}=\frac{\abr{||\hat{f}-f^*||_2^2-||\hat{f}-f^*||_N^2}}{||\hat{f}-f^*||_2+||\hat{f}-f^*||_N}\le \frac{\abr{||\hat{f}-f^*||_2^2-||\hat{f}-f^*||_N^2}}{||\hat{f}-f^*||_2}.
\end{align*}
Using this bound, we have
\begin{align*}
    \abr{||\hat{f}-f^*||_2-||\hat{f}-f^*||_N}\le 2r_N+8\sqrt{\frac{2\log(\varphi_{\lambda}r_N)/\delta)}{N}}+32\frac{\log(\varphi_{\lambda}r_N)/\delta)}{||\hat{f}-f^*||_2 N}
\end{align*}
Because we condition on that $||\hat{f}-f^*||_2\ge r_N$, then
\[
\frac{\log(\varphi_{\lambda}r_N)/\delta)}{||\hat{f}-f^*||_2 N}\le \frac{\log(\varphi_{\lambda}r_N)/\delta)}{r_N N}.
\]
Therefore, we have with probability $1-\delta$,
\[
||\hat{f}-f^*||_N\le ||\hat{f}-f^*||_2+2r_N+8\sqrt{\frac{2\log(\varphi_{\lambda}r_N)/\delta)}{N}}+32\frac{\log(\varphi_{\lambda}r_N)/\delta)}{r_N N}.
\]
For the term $\sqrt{\frac{1}{n}\sum_{i=m+1}^{m+n}\rbr{f^*(x_i,z_i^*,a_i)-f^*(x_i,\Tilde{z}_i,a_i)}^2}$, notice that
\[
|\sqrt{x}-\sqrt{y}|=\abr{\frac{x-y}{\sqrt{x}+\sqrt{y}}}\le \abr{\frac{x-y}{\sqrt{y}}}\ \text{for}\ x,y>0.
\]
By \cref{lemma:Hoeffding_Lipschitz}, we have that with probability at least $1-\delta$,
\[
\abr{\frac{1}{n}\sum_{i=m+1}^{m+n}\rbr{f^*(x_i,z_i^*,a_i)-f^*(x_i,\Tilde{z}_i,a_i)}^2-\EE[\rbr{f^*(x,z^*,a)-f^*(x,\Tilde{g}(x),a)}^2]}\le 4\sqrt{\frac{\log(2/\delta)}{n}}.
\]
Therefore, we have that with probability at least $1-\delta$,
\begin{align*}
&\sqrt{\frac{1}{n}\sum_{i=m+1}^{m+n}\rbr{f^*(x_i,z_i^*,a_i)-f^*(x_i,\Tilde{z}_i,a_i)}^2}\\
\le& \sqrt{\EE[\rbr{f^*(x,z^*,a)-f^*(x,\Tilde{g}(x),a)}^2]}+\frac{4}{\sqrt{\EE[\rbr{f^*(x,z^*,a)-f^*(x,\Tilde{g}(x),a)}^2]}}\sqrt{\frac{\log(2/\delta)}{n}}\\
\le &\rbr{\sqrt{\cE^{\cF}_{\PP}(\tilde{g})}+\frac{4}{\upsilon(f^*,\PP,\tilde{g})^{1/2}}\sqrt{\frac{\log(2/\delta)}{n}}}.
\end{align*}
Thus, we get that with probability at least $1-\delta$,
\[
\frac{K_3}{2}\le \rbr{\sqrt{\cE^{\cF}_{\PP}(\tilde{g})}+\frac{4}{\upsilon(f^*,\PP,\tilde{g})^{1/2}}\sqrt{\frac{\log(2/\delta)}{N}}}\rbr{||\hat{f}-f^*||_2+2r_N+8\sqrt{\frac{2\log(\varphi_{\lambda}r_N)/\delta)}{N}}+32\frac{\log(\varphi_{\lambda}r_N)/\delta)}{r_N N}}.
\]
So we finish the proof
\end{proof}

Now we are ready to prove \cref{thm:erm_oracle_ineq}.
\begin{proof}[Proof of \cref{thm:erm_oracle_ineq}]

Define $$A_N=\rbr{\max\cbr{4,\lambda}r_N+\max\cbr{16,4\lambda}\sqrt{\frac{2\log(\varphi_{\lambda}r_N)/\delta)}{N}}+\frac{64\log(\varphi_{\lambda}r_N)/\delta)}{r_NN}}.$$
From \cref{lemma:bound_K_1}, \cref{lemma:bound K_2} and \cref{lemma:bound_K_3}, with probability at least $1-4\delta$, we either have $||\hat{f}-f^*||_2\le r_N$ or
\begin{align*}
||f^*-\hat{f}||_2^2\le& ||\hat{f}-f^*||_2\rbr{2A_N+2\sqrt{\cE_{\PP}^{\cF}(\tilde{g})}+\frac{8}{\sqrt{\upsilon(f^*,\PP,\tilde{g})}}\sqrt{\frac{\log(2/\delta)}{N}}}\\
+&\rbr{\sqrt{\cE_{\PP}^{\cF}(\tilde{g})}+\frac{4}{\upsilon(f^*,\PP,\tilde{g})^{1/2}}\sqrt{\frac{\log(2/\delta)}{N}}}A_N.
\end{align*}
Rearranging the terms, we have
\begin{align*}
&\rbr{||\hat{f}-f^*||_2-\rbr{A_N+\sqrt{\cE_{\PP}^{\cF}(\tilde{g})}+\frac{4}{\upsilon(f^*,\PP,\tilde{g})^{1/2}}\sqrt{\frac{\log(2/\delta)}{N}}}}^2\\
\le& \rbr{\sqrt{\cE_{\PP}^{\cF}(\tilde{g})}+\frac{4}{\upsilon(f^*,\PP,\tilde{g})^{1/2}}\sqrt{\frac{\log(2/\delta)}{N}}}A_N+\rbr{A_N+\sqrt{\cE_{\PP}^{\cF}(\tilde{g})}+\frac{4}{\upsilon(f^*,\PP,\tilde{g})^{1/2}}\sqrt{\frac{\log(2/\delta)}{N}}}^2\\
\le& \frac{5}{4}\rbr{A_N+\sqrt{\cE_{\PP}^{\cF}(\tilde{g})}+\frac{4}{\upsilon(f^*,\PP,\tilde{g})^{1/2}}\sqrt{\frac{\log(2/\delta)}{N}}}^2.
\end{align*}
Therefore, for any $\delta>0$, with probability at least $1-4\delta$,
\begin{align*}
    ||\hat{f}-f^*||_2\le \rbr{1+\frac{\sqrt{5}}{2}}\rbr{A_N+\sqrt{\cE_{\PP}^{\cF}(\tilde{g})}+\frac{4}{\upsilon(f^*,\PP,\tilde{g})^{1/2}}\sqrt{\frac{\log(2/\delta)}{N}}}+r_N.
\end{align*}
Plugging in the definition of $A_N$, we shall finish the proof.
\end{proof}
Now, we are ready to prove the statistical regret guarantee in \cref{thm:regret}. To do that, we first introduce the following lemma from \citet{simchi2020bypassing} to transform statistical guarantees of regression oracles to regret guarantees of decision-making algorithms.
\begin{lemma}\citep{simchi2020bypassing}\label{lemma:general_guarantee_IGW}
Assume that we are given an offline regression oracle $\mathsf{RegOff}$ and i.i.d. data $\cD=\cbr{(x_i,a_i,r_i)}_{i=1}^{n}$ where $\EE[r_i|x_i,a_i]=f^*(x_i,a_i)$. With probability at least $1-\delta$, it returns
    $\hat{f}:\cX\times\cA\rightarrow \RR$ such that
    $$\EE_{x,a}\sbr{\rbr{\hat{f}(x,a)-f^*(x,a)}^2}\le\mathsf{Est}_{\delta}(n)$$
for some number $\mathsf{Est}_{\delta}(n)$. Then, define epoch schedule $\beta_s=2^s$ and exploration parameter $\gamma_s=\frac{1}{2}\sqrt{\frac{K}{\mathsf{Est}_{\delta/2s^2}(\beta_{s-1}-\beta_{s-2})}}$. For any $t$, let $s(t)$ be the number of epochs that round $t$ lies in. For any $T$ large enough, with probability at least $1-\delta$, the regret of applying the inverse gap weighting policy after $T$ rounds is at most 
$$\cO\rbr{\sqrt{K}\sum_{s=2}^{s(T)}\sqrt{\mathsf{Est}_{\delta/(2s^2)}(\beta_{s-1}-\beta_{s-2})}(\beta_s-\beta_{s-1})}.$$
\end{lemma}
\begin{proof}[Proof of \cref{thm:regret}]
We would like to apply \cref{lemma:general_guarantee_IGW}, so we first specify the number $\sqrt{\mathsf{Est}_{\delta}(n)}$ in our algorithm. By \cref{thm:erm_oracle_ineq}, we have
\[
\sqrt{\mathsf{Est}_{\delta}}(n)=C(\lambda,\upsilon(f^*,\PP,\tilde{g}))\rbr{r_n+\sqrt{\cE^{\cF}_{\PP}(\tilde{g})}+\sqrt{\frac{\log(4\varphi_{\lambda}(r_n)/\delta)}{n}}+\sqrt{\frac{\log(8/\delta)}{n}}+\frac{\log(\varphi_{\lambda}(r_n)/\delta)}{r_nn}}.
\]
Plugging this bound in, we have that
\begin{align*}
    \text{Reg}(T)\lesssim & \sqrt{K}\lbrace\sum_{s=2}^{s(T)}r_{\beta_{s-1}-\beta_{s-2}}(\beta_s-\beta_{s-1})+\sqrt{\log(4s^2\varphi_{\lambda}(r_{\beta_{s-1}-\beta_{s-2}})/\delta)+\log(8s^2/\delta)}\frac{\beta_s-\beta_{s-1}}{\sqrt{\beta_{s-1}-\beta_{s-2}}}\\
    &+\sum_{s=2}^{s(T)}\log(s^2\varphi_{\lambda}(4r_{\beta_{s-1}-\beta_{s-2}})/\delta)\frac{\beta_{s}-\beta_{s-1}}{r_{\beta_{s-1}-\beta_{s-2}}(\beta_{s-1}-\beta_{s-2})}+\sqrt{\cE^{\cF}_{\PP}(\tilde{g})}T\rbrace.
\end{align*}
Recall that we set $\beta_{s}=2^{s}$ in \cref{thm:regret}, then we have that
\begin{align*}
    \text{Reg}(T)\lesssim& \sqrt{K}\Big\lbrace\sum_{s=2}^{s(T)}r_{2^{s-2}}2^{s-1}+\sum_{s=2}^{s(T)}\sqrt{\log(4s^2\varphi_{\lambda}(r_{2^{s-2}})/\delta)+\log(8s^2/\delta)}2^{\frac{s}{2}}\\
    +&\sum_{s=2}^{s(T)}\log(\varphi_{\lambda}(4r_{2^{s-2}})/\delta)\frac{2}{r_{2^{s-2}}}+\sqrt{\cE^{\cF}_{\PP}(\tilde{g})}T\Big\rbrace.
\end{align*}
Thus, we finish the proof.
\end{proof}
\section{Proofs in \cref{sec:model_calibration}}
\begin{lemma}\label{lemma:hat{e}_estimation}
    Denote $s_N$ as the critical radius of the function class $\cT$ with dataset cardinality $N$, then for any $\delta\in(0,1)$, with probability at least $1-\delta$, our calibration procedure $\cref{alg:calibration_procedure}$ outputs $\hat{e}$ that satisfies
    \[
    ||\hat{e}-e^*||_2\le \frac{5}{2}s_{N/2}+10\sqrt{\frac{2\log(\varphi_2(s_{N/2})/\delta)}{N/2}},
    \]
    where $\varphi_2(s_{N/2})=\log_2(\frac{4}{s_{N/2}})$ as pre-specified in \cref{subsec:ERM Oracle Inequality with Imputation}.
\end{lemma}
\begin{proof}[Proof of \cref{lemma:hat{e}_estimation}]
    We first elaborate on the estimation of $\hat{e}$. Note that $$||\hat{e}-e^*||_2^2=||\hat{e}-e^*||_2^2-||\hat{e}-e^*||_{N/2}^2+||\hat{e}-e^*||_{N/2}^2.$$ 
We define the local Rademacher complexity of $\cT$ as $\cR_{N/2}(t,\cT)$ and the critical radius of $\cT$ as $s_{N/2}$.  

For the term $||\hat{e}-e^*||_2^2-||\hat{e}-e^*||_{N/2}^2$, we could apply \cref{thm: sup_empirical_process} by defining $H_{e}(x):=(e(x)-e^*(x))^2-||e-e^*||_2^2$ and $\cH(t):=\cbr{H_e(x):e\in\cT, ||e-e^*||_2\le t}$. We also define $$Z_{N/2}(t):=\sup_{e\in\cH(t)}\cbr{\PP_{N/2}(H_{e})-\PP(H_{e})}.$$ Then with probability at least $1-\delta$, we have
\[
Z_{N/2}(t)\le 2\EE[Z_{N/2}(t)]+\sqrt{\sigma^2(\cH(t))}\sqrt{\frac{2\log(1/\delta)}{N/2}}+2\frac{\log(1/\delta)}{N/2}.
\]
We first have $\sigma^2(\cH(t))\le \PP((e-e^*)^4)\le 4t^2$. For $\EE[Z_{N/2}(t)]$, we have
\[
\EE[Z_{N/2}(t)]=\EE\sbr{\sup_{||e^*-e||_2\le t}\abr{||e^*-e||_{N/2}^2-||e-e^*||_2^2}}\le 2\EE\sbr{\sup_{||e^*-e||_2\le t}\abr{\frac{2}{N}\sum_{i=1}^{N/2}\varepsilon_i\rbr{e^*(x_i)-e(x_i)^2}}}.
    \]
    By \cref{lemma:Ledoux_Talagrand}, the function $\phi_i(t)=\phi(t)=t^2$ is $2$-Lipschitz on the interval $[0,1]$, we have
    \begin{align*}\hspace{-0.6cm}
    2\EE\sbr{\sup_{||e^*-e||_2\le t}\abr{\frac{2}{N}\sum_{i=1}^{N/2}\varepsilon_i\rbr{e^*(x_i)-e(x_i)}^2}}\le 4\EE\sbr{\sup_{||e^*-e||_2\le t}\abr{\frac{2}{N}\sum_{i=1}^{N/2}\varepsilon_i\rbr{e^*(x_i)-e(x_i)}}}=4\cR_{N/2}(t,\cT).
    \end{align*}
    By definition, we have $\cR_{N/2}(t,\cT)\le\frac{t}{16}s_{N/2}$.
    Plugging all these in, we get with probability at least $1-\delta$,
    \[
    Z_{N/2}(t)\le \frac{s_{N/2}}{2}t+4\sqrt{\frac{2\log(1/\delta)}{N/2}}t+\frac{2\log(1/\delta)}{N/2}.
    \]
Apply the peeling argument in \cref{lemma:peeling}, we further have that with probability at least $1-\delta$,
\[
Z_n(||\hat{e}-e^*||_2)\le s_{N/2}||\hat{e}-e^*||_2+8\sqrt{\frac{2\log(\varphi_{2}(s_{N/2})/\delta)}{N/2}}||\hat{e}-e^*||_2+4\frac{\log(\varphi_{2}(s_{N/2})/\delta)}{n}.
\]
For the second term $||\hat{e}-e^*||_{N/2}^2$, conditioned on $||\hat{e}-e^*||_2\ge s_{N/2}$, we first apply the property of the ERM to get
\[
||\hat{e}-e^*||_{N/2}^2\le \frac{4}{N}\sum_{i=1}^{N/2}(b_i-e^*(x_i))(\hat{e}(x_i)-e^*(x_i)).
\]
Conditioned on $||\hat{e}-e^*||\ge s_{N/2}$, we define the empirical process 
\[
W_{N/2}(t):=\sup_{||e-e^*||_2\le t}\abr{\frac{2}{N}\sum_{i=1}^{N/2}(b_i-e^*(x_i))(\hat{e}(x_i)-e^*(x_i))}
\]
Denote $b_i-e^*(x_i)$ as $r_i$ and define the function class $\cH(t):=\cbr{r(e-e^*):e\in\cT,||e-e^*||_2\le t}$. Then we apply \cref{thm: sup_empirical_process} to get that with probability at least $1-\delta$,
\[
W_{N/2}(t)\le 2\EE[W_{N/2}(t)]+\sqrt{\sigma^2(\cH(t))}\sqrt{\frac{2\log(1/\delta)}{N/2}}+2\frac{\log(1/\delta)}{N/2}
\]
Similarly, we have $\sigma^2(\cH(t))\le t^2$ and 
\[
\EE[W_{N/2}(t)]\le 2\cR_{N/2}(t,\cT).
\]
By definition, for $t\ge s_{N/2}$, we have $\cR_{N/2}(t,\cT)\le \frac{s_{N/2}}{16}$. Then with probability at least $1-\delta$,
\[
W_{N/2}(t)\le \frac{s_{N/2}}{4}t+\sqrt{\frac{2\log(1/\delta)}{N/2}}t+\frac{2\log(1/\delta)}{N/2}.
\]
By the peeling argument again, we have with probability at least $1-\delta$,
\[
W_{N/2}(||\hat{e}-e^*||_2)\le \frac{s_{N/2}}{2}||\hat{e}-e^*||_2+2\sqrt{\frac{2\log(\varphi_2(s_{N/2})/\delta)}{n}}||\hat{e}-e^*||_2+\frac{4\log(\varphi_2(s_{N/2})/\delta)}{N/2}.
\]
Now we are ready to prove the bound, by construction, we either have $||e^*-\hat{e}||_2\le s_{N/2}$ or with probability $1-\delta$,
\[
||\hat{e}-e^*||_2^2\le ||\hat{e}-e^*||_2\rbr{\frac{3}{2}s_{N/2}+10\sqrt{\frac{2\log(\varphi_2(s_{N/2})/\delta)}{N/2}}}+6\frac{\log(\varphi_2(s_{N/2})/\delta)}{N/2}.
\]
Thus, with probability at least $1-\delta$,
\[
||\hat{e}-e^*||_2\le \frac{5}{2}s_{N/2}+10\sqrt{\frac{4\log(\varphi_2(s_{N/2})/\delta)}{N}}.
\]
\end{proof}
We are now ready to analyze the statistical error of $\hat{g}$. We first define the point-wise loss $$l(g,e,x,b,z):=\frac{b(bz-g(x))^2}{e(x)},$$
where $g\in\cG$, and $e\in\cT$. Moreover, $b\in\cbr{0,1}$ indicates the missingness. Correspondingly, we define $\cL(g,e):=\EE_{x,b,z}[l(g,e,x,b,z)]$. Then, it is easy to verify that 
\[
\cL(g,e^*)=\EE[(z-g(x))^2]=||g-g^*||_2^2+\EE[\eta^2].
\]
To study the property of $\cL(g,e)$, we consider the Gâteaux derivative or the Fréchet derivative in variational analysis \citep{nashed1966some,abbasi2021gateaux} to carry out the Taylor expansion of it. By some algebra, we have the following lemma.
\begin{lemma}\label{lemma:Gateaux_derivative}
    By computing the G\^{a}teaux derivatives, we have that
    \[
    D_{g}\cL(g,e)[h]=2\EE\sbr{\frac{e^*(g-g^*)}{e}h}.
    \]
    \[
    D_{g}^2\cL(g,e)[h,h]=2\EE[\frac{e^*}{e}h^2].
    \]
    \[
    D_{e}D_g\cL(g,e)[h_1,h_2]=-2\EE\sbr{\frac{e^*(g-g^*)h_1h_2}{e^2}}.
    \]
    \[
    D_{e}^2D_g\cL(g,e)[h_1,h_2,h_2]=4\EE\sbr{\frac{e^*(g-g^*)h_1h_2^2}{e^3}}.
    \]
\end{lemma}
\begin{proof}[Proof of \cref{thm:calibration_guarantee}]
We first apply a Taylor expansion of the expected loss function $\cL(g,e)$ at point $(g^*,\hat{e})$ to evaluate the value of $\cL(\hat{g},\hat{e})$. $\exists \bar{g}$ such that
\begin{align*}
    \cL(\hat{g},\hat{e})=\cL(g^*,\hat{e})+D_g\cL(g^*,\hat{e})[\hat{g}-g^*]+\frac{1}{2}D_{g}^2\cL(\bar{g},\hat{e})[\hat{g}-g^*,\hat{g}-g^*].
\end{align*}
Applying \cref{lemma:Gateaux_derivative}, we have that
\begin{align*}
    D_g^2\cL(\bar{g},\hat{e})[\hat{g}-g^*,\hat{g}-g^*]=\EE\sbr{\frac{e^*-\hat{e}+\hat{e}}{\hat{e}}(g-g^*)^2}=\EE[(g-g^*)^2]+\EE[\frac{e^*-\hat{e}}{\hat{e}}(g-g^*)^2].
\end{align*}
We need to bound the second term and we do this by applying the Cauchy-Schwarz inequality. 
\begin{align*}
    \abr{\EE[\frac{e^*-\hat{e}}{\hat{e}}(g-g^*)^2]}\le \frac{1}{\epsilon_0}(\frac{1}{\alpha}\EE[(e^*-\hat{e})^2]+\alpha\EE[(g^*-g)^4])\le\frac{4\alpha}{\epsilon_0}\EE[(g^*-g)^2]+\frac{1}{\alpha\epsilon_0}||\hat{e}-e^*||_2^2.
\end{align*}
Thus, by taking $\alpha=\frac{\epsilon_0}{8}$, we have
\[
D_g^2\cL(\bar{g},\hat{e})[\hat{g}-g^*,\hat{g}-g^*]\ge \frac{1}{2}||\hat{g}-g^*||_2^2-\frac{8}{\epsilon_0^2}||\hat{e}-e^*||_2^2.
\]
For the term \( D_g\cL(g^*,\hat{e})[\hat{g}-g^*]\), we apply \cref{lemma:Gateaux_derivative} and use the fact that $\EE[\eta|x]=0$ to get that $D_g\cL(g^*,\hat{e})[\hat{g}-g^*]=0$.
Finally, we have,\[
\frac{1}{2}||\hat{g}-g^*||_2^2\le \frac{8}{\epsilon_0^2}||\hat{e}-e^*||_2^2+\cL(\hat{g},\hat{e})-\cL(g^*,\hat{e}).
\]
Therefore, we only need to bound $\cL(\hat{g},\hat{e})-\cL(g^*,\hat{e})$.
Recall that we minimize the empirical loss
\[
\hat{g}=\argmin_{g\in\cG_{\delta_0}}\frac{2}{N}\sum_{i=N/2+1}^{N}\frac{b_i(b_iz_i^*-g(x_i))^2}{\hat{e}(x_i)},
\]
where we denote the empirical loss by $\cL_{N/2}(g,e)$. Then we have
\[
\cL(\hat{g},\hat{e})-\cL(g^*,\hat{e})=(\cL(\hat{g},\hat{e})-\cL_{N/2}(\hat{g},\hat{e}))+(\cL_{N/2}(\hat{g},\hat{e})-\cL_{N/2}(\hat{g},e^*))+(\cL_{N/2}(\hat{g},e^*)-\cL(\hat{g},e^*)).
\]
The middle term is non-positive by the definition of ERM. We now focus on the remaining two terms. The analysis of them is quite similar and we only focus on the first one. All the following mathematical deductions are conditioned on $||\hat{g}-g^*||_2\ge q_{\delta_0,N/2}$, 

Define $H_{g}(x,b,z;\hat{e}):=\frac{b(bz-g(x))^2-\tau^2}{\hat{e}(x)}$. $\cH(t):=\cbr{H_{g}:g\in\cG_{\delta_0},||g-g^*||_2\le t}$, Define \[
W_{N/2}(t)=\sup_{H_g\in\cH(t)}\cbr{\PP(H_g)-\PP_{N/2}(H_g)}.
\] 
By construction, we have
\[
\cL(\hat{g},\hat{e})-\cL_{N/2}(\hat{g},\hat{e})\le W_{N/2}(||\hat{g}-g^*||_2),\ \cL(g^*,\hat{e})-\cL_{N/2}(g^*,\hat{e})\le W_{N/2}(||\hat{g}-g^*||_2).
\]
Applying \cref{thm: sup_empirical_process}, we have that with probability at least $1-\delta$, 
\[
W_{N/2}(t)\le 2\EE[W_{N/2}(t)]+\sqrt{\sigma^2(\cH(t))}\sqrt{\frac{2\log(1/\delta)}{N/2}}+4\frac{(1+\tau)^2}{\epsilon_0}\frac{\log(1/\delta)}{N/2}.
\]
\[
\sigma^2(\cH(t))\le \frac{(1+\tau)^4}{\epsilon_0^2}t^2
\]
\begin{align*}
    &\EE[W_{N/2}(t)]\\
    \le& 2\EE\sbr{\sup_{g\in\cG_{\delta_0},||g-g^*||_2\le t}\frac{2}{N}\sum_{i=N/2+1}^{N}\frac{b_i(b_iz_i-g(x_i))^2-\tau^2}{\hat{e}(x_i)}}\\
    =&2\EE\sbr{\sup_{g\in\cG_{\delta_0},||g-g^*||_2\le t}\frac{2}{N}\sum_{i=N/2+1}^{N}\frac{b_i(z_i^2-2z_ig(x_i)+g(x_i)^2)-\tau^2}{\hat{e}(x_i)}}\\
    =&2\EE\sbr{\sup_{g\in\cG_{\delta_0},||g-g^*||_2\le t}\frac{2}{N}\sum_{i=N/2+1}^{N}\frac{b_i\rbr{(g-g^*)^2(x_i)+2\eta_i(g^*-g)(x_i)+\eta_i^2}-\tau^2}{\hat{e}(x_i)}}\\
    \le&2\EE\sbr{\sup_{g\in\cG_{\delta_0},||g-g^*||_2\le t}\frac{2}{N}\sum_{i=N/2+1}^{N}\frac{b_i(g-g^*)^2(x_i)}{\hat{e}(x_i)}}+2\EE\sbr{\sup_{g\in\cG_{\delta_0},||g-g^*||_2\le t}\frac{2}{N}\sum_{i=N/2+1}^{N}\frac{2\eta_i(g^*-g)(x_i)}{\hat{e}(x_i)}}\\
    &+\EE[\frac{2}{N}\sum_{i=N/2+1}^{N}\frac{\eta_i^2-\tau^2}{\hat{e}(x_i)}]\\
    \le&2\frac{1+\tau}{\epsilon_0}\cR_{N/2}(t,\cG_{\delta_0}).
\end{align*}
We define the critical radius of $\cG_{\delta_0}$ as $q_{\delta_0,N/2}$, we have $\EE[W_{N/2}(t)]\le \frac{2}{\epsilon_0}q_{\delta_0,N/2}t$.
Thus,
with probability at least $1-\delta$, we have
\[
W_{N/2}(t)\le \frac{4(1+\tau)}{\epsilon_0}q_{\delta_0,N/2}t+\frac{(1+\tau)^2}{\epsilon_0}\sqrt{\frac{2\log(1/\delta)}{N/2}}t+4\frac{(1+\tau)^2}{\epsilon_0}\frac{\log(1/\delta)}{N/2}.
\]
By \cref{lemma:peeling}, with probability at least $1-\delta$, conditioned on $||\hat{g}-g^*||_2\ge q_{\delta_0,N/2}$,
\begin{align*}
W_{N/2}(||\hat{g}-g^*||_2)\le& \frac{8(1+\tau)}{\epsilon_0}q_{\delta_0,N/2}||\hat{g}-g^*||_2+\frac{2(1+\tau)^2}{\epsilon_0}\sqrt{\frac{2\log(\varphi_{1+\tau}(q_{\delta_0,N/2})/\delta)}{N/2}}||\hat{g}-g^*||_2\\
+&8\frac{(1+\tau)^2}{\epsilon_0}\frac{\log(\varphi_{1+\tau}(q_{\delta_0,N/2})/\delta)}{N/2}.    
\end{align*}

The bound for $\rbr{\cL_{N/2}(g^*,\hat{e})-\cL(g^*,\hat{e})}$ is the same. Therefore, we have that conditioned on $||\hat{g}-g^*||_2\ge q_{\delta_0,N/2}$, with probability at least $1-2\delta$,
\begin{align*}
    \frac{1}{2}||\hat{g}-g^*||_2^2\le&
    \frac{16(1+\tau)}{\epsilon_0}q_{\delta_0,N/2}||\hat{g}-g^*||_2+\frac{4(1+\tau)^2}{\epsilon_0}\sqrt{\frac{2\log(\varphi_{1+\tau}(q_{\delta_0,N/2})/\delta)}{N/2}}||\hat{g}-g^*||_2\\
    +&16\frac{(1+\tau)^2}{\epsilon_0}\frac{\log(\varphi_{1+\tau}(q_{\delta_0,N/2})/\delta)}{N/2}+\frac{8}{\epsilon_0^2}||\hat{e}-e^*||_2^2.\\
\end{align*}
By rearranging the terms and some algebra, we have that if $||\hat{g}-g^*||_2\ge q_{\delta_0,N/2}$, then with probability at least $1-2\delta$,
\begin{align*}
    ||\hat{g}-g^*||_2^2\le& 4\rbr{\frac{16(1+\tau)}{\epsilon_0}q_{\delta_0,N/2}+\frac{4(1+\tau)^2}{\epsilon_0}\sqrt{\frac{2\log(\varphi_{1+\tau}(q_{\delta_0,N/2})/\delta)}{N/2}}}^2\\
    +&\rbr{32\frac{(1+\tau)^2}{\epsilon_0}\frac{\log(\varphi_{1+\tau}(q_{\delta_0,N/2})/\delta)}{N/2}+\frac{16}{\epsilon_0^2}||\hat{e}-e^*||_2^2}.
\end{align*}

Now, we try to quantify the benefit of pre-trained model calibration by computing $q_{\delta_0,N/2}$, and thus illustrate the benefit of a good pre-trained model. We solve for a general  number $n$ and then plug in $n=N/2$ later.

In \cref{ass:calibration}, we assume that the covering number of $\cG$ satisfies that $\log\cN(\epsilon,\cG,||\cdot||_2)\lesssim \rbr{\frac{1}{\epsilon}}^d$. Assuming that we have found such a minimal cover $\cC$, we claim that
\[
\cC_{\delta_0}:=\cbr{g\in\cC,||g-\tilde{g}||_2\le \delta_0+\epsilon}
\]
is an $\delta_0+\epsilon$ cover of $\cG_{\delta_0}$. Therefore, if $\epsilon<\delta_0$ and $\delta_0$ is small, we have
\[
\log(\epsilon,\cG_{\delta_0},||\cdot||_2)\lesssim A(\frac{\epsilon+\delta_0}{\epsilon})^d\le 2^d(\frac{\delta_0}{\epsilon})^d\Rightarrow \log(\epsilon,\cG_{\delta_0},||\cdot||_2)\lesssim (\frac{\delta_0}{\epsilon})^d.
\]
Applying Dudley's integral bound, we have
\[
\mathfrak{R}_n(t,\mathcal{\cG}_{\delta_0}) \lesssim \inf_{\alpha > 0} \left\{ 4\alpha + \frac{12}{\sqrt{n}} \int_\alpha^{t} \sqrt{ \log \mathcal{N}(\varepsilon, \mathcal{\cG}_{\delta_0}, L_2(\PP)) } \, d\varepsilon \right\}.
\]
Using the bound
\[
\log \mathcal{N}(\varepsilon, \mathcal{\cG}_{\delta_0}, L_2(\PP)) \lesssim \left( \frac{\delta_0}{\varepsilon} \right)^d,
\]
we get
\[
\sqrt{ \log \mathcal{N}(\varepsilon, \mathcal{\cG}_{\delta_0}, L_2(\PP)) } \lesssim \left( \frac{\delta_0}{\varepsilon} \right)^{d/2}.
\]
Substituting into the Dudley integral \citep{koltchinskii2011oracle}:
\[
\int_\alpha^{\delta_0} \left( \frac{\delta_0}{\varepsilon} \right)^{d/2} d\varepsilon = {\delta_0}^{d/2} \int_\alpha^{\delta_0} \varepsilon^{-d/2} d\varepsilon.
\]
Without loss of generality, we assume that $d\neq 2$. Then
\[
\int_\alpha^{t} \varepsilon^{-d/2} d\varepsilon = \frac{1}{1 - d/2} \left( {t}^{1 - d/2} - \alpha^{1 - d/2} \right).
\]
Hence,
\[
\mathfrak{R}_n(t,\cG_{\delta_0}) \lesssim \alpha + \frac{12{\delta_0}^{d/2}}{(1-d/2)\sqrt{n}} \cdot \left( {t}^{1 - d/2} - \alpha^{1 - d/2} \right).
\]
Balancing between the two terms, we take $\alpha_{opt}\asymp\delta_0^{\frac{d}{d+2}}n^{-\frac{1}{d+2}}$.

Recall the definition of the critical radius, we need 
\[
\alpha_{opt}\lesssim q_{\delta_0,n}^2,\ \frac{\delta_0^{d/2}}{\sqrt{n}}q_{\delta_0,n}^{1-d/2}\lesssim q_{\delta_0,n}^2.
\]
Thus, we get
\[
q_{\delta_0,n}\asymp \delta_0^{\frac{d}{d+2}}n^{\frac{-1}{d+2}}.
\]
Plugging in $n=N/2$ and combining all these parts together, we finish the proof.
\end{proof}
Now, we are ready to give the proof of \cref{thm:regret_calibration}. 
\begin{proof}[Proof of \cref{thm:regret_calibration}]
In epoch $s$, the data used for fitting the reward function $\hat{f}_s$ is $$\cD_{s-1}=\cbr{(x_t,b_t,\tilde{z}_t,a_t,r_t)}_{t=\beta_{s-2}+1}^{\beta_{s-1}},\ |\cD_{s-1}|=2^{s-2}.$$
In $\cD_{s-1}$, $\tilde{z}_t$ is generated by $\hat{g}_{s-1}$ which is trained by the calibration procedure $\Cal$ on the dataset $\cD_{s-2}$ from epoch $s-2$. Therefore, combining \cref{thm:erm_oracle_ineq} and \cref{thm:calibration_guarantee}, we have that for any $\delta\in(0,1)$, with probability at least $1-3\delta$,
\begin{align*}
    ||\hat{f}_{s}-f^*||_2\le &C(\lambda,\upsilon(f^*,\PP,\tilde{g}))\rbr{r_{2^{s-2}}+\sqrt{\cE^{\cF}_{\PP}(\hat{g}_{s-1})}+\sqrt{\frac{\log(4\varphi_{\lambda}(r_{2^{s-2}})/\delta)}{2^{s-2}}}+\sqrt{\frac{\log(8/\delta)}{2^{s-2}}}+\frac{\log(\varphi_{\lambda}(r_{2^{s-2}})/\delta)}{r_{2^{s-2}}2^{s-2}}}\\
    \lesssim & C(\lambda,\upsilon(f^*,\PP,\tilde{g}))\rbr{r_{2^{s-2}}+\sqrt{\frac{\log(4\varphi_{\lambda}(r_{2^{s-2}})/\delta)}{2^{s-2}}}+\sqrt{\frac{\log(8/\delta)}{2^{s-2}}}+\frac{\log(\varphi_{\lambda}(r_{2^{s-2}})/\delta)}{r_{2^{s-2}}2^{s-2}}}\\
    +&L_{\cZ}||\hat{g}_{s-1}-g^*||_2+L_{\cZ}\EE[\eta^2]^{\frac{1}{2}}\\
    \lesssim & C(\lambda,\upsilon(f^*,\PP,\tilde{g}))\rbr{r_{2^{s-2}}+\sqrt{\frac{\log(4\varphi_{\lambda}(r_{2^{s-2}})/\delta)}{2^{s-2}}}+\sqrt{\frac{\log(8/\delta)}{2^{s-2}}}+\frac{\log(\varphi_{\lambda}(r_{2^{s-2}})/\delta)}{r_{2^{s-2}}2^{s-2}}}\\
    +&L_{\cZ}\sqrt{\frac{1}{\epsilon_0^2}\rbr{\delta_0^{\frac{2d}{d+2}}N^{\frac{-2}{d+2}}+\frac{(1+\tau)^2\log(\varphi_{1+\tau}(q_{\delta_0,N/2})/\delta)}{N}+s_{N/2}^2+\frac{\log(\varphi_2(s_{N/2})/\delta)}{N}}}+L_{\cZ}\omega_0\\
    \lesssim & C(\lambda,\upsilon(f^*,\PP,\tilde{g}))\rbr{r_{2^{s-2}}+\sqrt{\frac{\log(4\varphi_{\lambda}(r_{2^{s-2}})/\delta)}{2^{s-2}}}+\sqrt{\frac{\log(8/\delta)}{2^{s-2}}}+\frac{\log(\varphi_{\lambda}(r_{2^{s-2}})/\delta)}{r_{2^{s-2}}2^{s-2}}}\\
    +&\frac{L_{\cZ}}{\epsilon_0}\rbr{\delta_0^{\frac{d}{d+2}}2^{\frac{-2(s-3)}{d+2}}+\frac{(1+\tau)\sqrt{\log(\varphi_{1+\tau}(q_{\delta_0,2^{s-3}})/\delta)}}{2^{s/2-2}}+s_{2^{s-3}}+\frac{\sqrt{\log(\varphi_2(s_{2^{s-3}})/\delta)}}{2^{s/2-2}}}+L_{\cZ}\omega_0.
\end{align*}
The first inequality is due to \cref{thm:erm_oracle_ineq}, and the second inequality is by the Lipschitz continuity of the reward model class in \cref{ass:calibration}. The third one is by \cref{thm:calibration_guarantee}, and the last one is simply algebra.
Combining this bound, \cref{lemma:general_guarantee_IGW} and \cref{thm:regret}, we have that with probability at least $1-3\delta$,
\begin{align*}
    \text{Reg}(T)\lesssim &\sqrt{K}\rbr{\sum_{s=2}^{s(T)}r_{2^{s-2}}2^{s-1}+\sum_{s=2}^{s(T)}\sqrt{\log(4\varphi_{\lambda}(r_{2^{s-2}})/\delta)+\log(8/\delta)}2^{\frac{s}{2}}+\sum_{s=2}^{s(T)}\log(\varphi_{\lambda}(4r_{2^{s-2}})/\delta)\frac{2}{r_{2^{s-2}}}}\\
    +&\sqrt{K}\sum_{s=2}^{s(T)}\frac{L_{\cZ}}{\epsilon_0}\rbr{\delta_0^{\frac{d}{d+2}}2^{\frac{-2(s-3)}{d+2}}+\frac{(1+\tau)\sqrt{\log(\varphi_{1+\tau}(q_{\delta_0,2^{s-3}})/\delta)}}{2^{s/2-2}}+s_{2^{s-3}}+\frac{\sqrt{\log(\varphi_2(s_{2^{s-3}})/\delta)}}{2^{s/2-2}}}2^{s-1}\\
    &+L_{\cZ}\omega_0 T\\
    \lesssim& \sqrt{K}\rbr{\sum_{s=2}^{s(T)}r_{2^{s-2}}2^{s-1}+\frac{L_{\cZ}}{\epsilon_0}\delta_0^{\frac{d}{d+2}}\sum_{s=2}^{s(T)}2^{\frac{-2(s-3)}{d+2}+s-1}+\frac{L_{\cZ}}{\epsilon_0}\sum_{s=2}^{s(T)}s_{2^{s-3}}2^{s-1}}+L_{\cZ}\omega_0T\\
    +&\sqrt{K}\rbr{\sum_{s=2}^{s(T)}\sqrt{\log(4\varphi_{\lambda}(r_{2^{s-2}})/\delta)+\log(8/\delta)}2^{\frac{s}{2}}+\sum_{s=2}^{s(T)}\log(\varphi_{\lambda}(4r_{2^{s-2}})/\delta)\frac{2}{r_{2^{s-2}}}}\\
    +&\frac{\sqrt{K}L_{\cZ}}{\epsilon_0}\rbr{\sqrt{(1+\tau)^2\log(\varphi_{1+\tau}(q_{\delta_0,2^{s-3}})/\delta)+\log(\varphi_2(s_{2^{s-3}})/\delta)}}2^{s/2}.
\end{align*}
So we finish the proof.
\end{proof}
\end{document}